\documentclass{article}

\usepackage{hyperref}

\usepackage[accepted]{icml2024}

\usepackage{amsmath}
\usepackage{amssymb}
\usepackage{mathtools}
\usepackage{amsthm}

\usepackage[capitalize,noabbrev]{cleveref}

\theoremstyle{plain}
\newtheorem{theorem}{Theorem}

\newtheorem{lemma}[theorem]{Lemma}

\theoremstyle{definition}
\newtheorem{definition}[theorem]{Definition}
\newtheorem{assumption}[theorem]{Assumption}
\theoremstyle{remark}
\newtheorem{remark}[theorem]{Remark}

\usepackage[textsize=tiny]{todonotes}

\usepackage{authblk}
\usepackage{comment}
\usepackage{amsmath}
\usepackage{amsthm}
\usepackage{arydshln}
\usepackage{bigstrut}
\usepackage{bm}
\usepackage{booktabs}
\usepackage{color}
\usepackage{enumitem}
\usepackage{mathtools}
\usepackage{multirow}
\usepackage{multicol}
\usepackage{rotating}
\usepackage{stmaryrd}
\usepackage{subfigure}
\usepackage{tablefootnote}
\usepackage{longtable}
\usepackage{threeparttable}
\usepackage{amssymb}
\usepackage{pifont}
\newcommand{\cmark}{\ding{51}}%
\newcommand{\xmark}{\ding{55}}%
\usepackage{wrapfig}
\newcommand*\diff{\mathop{}\!\mathrm{d}}

\usepackage{makecell}
\icmltitlerunning{Learning with Complementary Labels Revisited: The Selected-Completely-at-Random Setting Is More Practical}

\begin{document}

\twocolumn[
\icmltitle{Learning with Complementary Labels Revisited: The Selected-Completely-at-Random Setting Is More Practical}


\icmlsetsymbol{equal}{*}

\begin{icmlauthorlist}
\icmlauthor{Wei Wang}{utokyo,riken}
\icmlauthor{Takashi Ishida}{riken,utokyo}
\icmlauthor{Yu-Jie Zhang}{utokyo}
\icmlauthor{Gang Niu}{riken}
\icmlauthor{Masashi Sugiyama}{riken,utokyo}
\end{icmlauthorlist}

\icmlaffiliation{utokyo}{The University of Tokyo}
\icmlaffiliation{riken}{RIKEN}

\icmlcorrespondingauthor{Masashi Sugiyama}{sugi@k.u-tokyo.ac.jp}

\icmlkeywords{Complementary-label learning, weakly supervised learning, positive-unlabeled learning}

\vskip 0.3in
]



\printAffiliationsAndNotice{}  

\begin{abstract}
Complementary-label learning is a weakly supervised learning problem in which each training example is associated with one or multiple complementary labels indicating the classes to which it does not belong. Existing consistent approaches have relied on the uniform distribution assumption to model the generation of complementary labels, or on an ordinary-label training set to estimate the transition matrix in non-uniform cases. However, either condition may not be satisfied in real-world scenarios. In this paper, we propose a novel consistent approach that does not rely on these conditions. Inspired by the positive-unlabeled~(PU) learning literature, we propose an \emph{unbiased risk estimator} based on the \emph{Selected-Completely-at-Random assumption} for complementary-label learning. We then introduce a \emph{risk-correction approach} to address overfitting problems. Furthermore, we find that complementary-label learning can be expressed as a set of \emph{negative-unlabeled binary classification} problems when using the one-versus-rest strategy. Extensive experimental results on both synthetic and real-world benchmark datasets validate the superiority of our proposed approach over state-of-the-art methods. 
\end{abstract}
\section{Introduction}
Complementary-label learning is a weakly supervised learning problem that has received a lot of attention recently~\citep{ishida2017learning,feng2020learning,gao2021discriminative,liu2023consistent}. 
In complementary-label learning, we are given training data associated with complementary labels that specify the classes to which the examples do not belong. The task is to learn a multi-class classifier that assigns correct labels to test data as in the standard supervised learning. Collecting training data with complementary labels is much easier and cheaper than collecting ordinary-label data. For example, when asking workers on crowdsourcing platforms to annotate training data, we only need to randomly select a candidate label and then ask them whether the example belongs to that class or not. Such ``yes'' or ``no'' questions are much easier to answer than asking workers to determine the ground-truth label from a large set of candidate labels. The benefits and effectiveness of complementary-label learning have also been demonstrated in several machine learning problems and applications, such as domain adaptation~\citep{zhang2021learning,han2023rethinking}, semi-supervised learning~\citep{chen2020negative,ma2023rethinking,deng2024boosting},
noisy-label learning~\citep{kim2019nlnl}, adversarial robustness~\citep{zhou2022adversarial}, few-shot learning~\citep{wei2022an}, and medical image analysis~\citep{rezaei2020recurrent}.
\begin{table*}[t]
\center
  \caption{Comparison between SCARCE and previous risk-consistent or classifier-consistent complementary-label learning methods.}
  \label{comparison}
  \begin{tabular}{lcccc}
\toprule[1pt]
\multirow{2}{*}{Method} & \multirow{2}{*}{\shortstack{Uniform distribution\\assumption-free}} & \multirow{2}{*}{\shortstack{Ordinary-label \\ training set-free}} & \multirow{2}{*}{\shortstack{Classifier-\\consistent}} &  \multirow{2}{*}{\shortstack{Risk-\\consistent}} \\
&&&\\\midrule
PC~\citep{ishida2017learning}& \xmark & \cmark &\cmark& \cmark\\
Forward~\citep{yu2018learning} & \cmark & \xmark  &\cmark & \xmark \\
NN~\citep{ishida2019complementary}& \xmark & \cmark &\cmark& \cmark\\
LMCL~\citep{feng2020learning}& \xmark & \cmark &\cmark& \cmark\\
OP~\citep{liu2023consistent} & \xmark & \cmark &\cmark & \xmark \\
\midrule
SCARCE~(Ours) & \cmark &\cmark &\cmark &~~\cmark \\
\bottomrule[1pt]
  \end{tabular}
\end{table*}

Existing research works with \emph{consistency guarantees} have attempted to solve complementary-label learning problems by making assumptions about the distribution of complementary labels. The remedy started with~\citet{ishida2017learning}, which proposed the \emph{uniform distribution assumption} that a label other than the ground-truth label is sampled from the uniform distribution to be the complementary label. A subsequent work extended it to arbitrary loss functions and models~\citep{ishida2019complementary} based on the same distribution assumption. Then,~\citet{feng2020learning} extended the problem setting to the existence of multiple complementary labels. Recent works have proposed discriminative methods that work by modeling the posterior probabilities of complementary labels instead of the generation process~\citep{chou2020unbiased,gao2021discriminative,liu2023consistent,lin2023reduction}. However, the uniform distribution assumption is still necessary to ensure the classifier consistency property~\citep{liu2023consistent}.~\citet{yu2018learning} proposed the \emph{biased distribution assumption}, elaborating that the generation of complementary labels follows a \emph{transition matrix}, i.e., the complementary-label distribution is determined by the true label. 

In summary, previous complementary-label learning approaches all require either the uniform distribution assumption or the biased distribution assumption to guarantee the consistency property, to the best of our knowledge.  However, such assumptions may not be satisfied in real-world scenarios. On the one hand, the uniform distribution assumption is too strong, since the transition probability for different complementary labels is undifferentiated, i.e., the transition probability from the true label to a complementary label is constant for all labels. Such an assumption is not realistic since the annotations may be imbalanced and biased~\citep{wei2023class,wang2023clcifar}. On the other hand, although the biased distribution assumption is more practical, an ordinary-label training set with \emph{deterministic labels}, also known as \emph{anchor points}~\citep{liu2015classification}, is essential for estimating transition probabilities during the training phase~\citep{yu2018learning}. However, the collection of ordinary-label data with deterministic labels is often unrealistic in complementary-label learning problems~\citep{feng2020learning,gao2021discriminative}. 

To this end, we propose a novel risk-consistent approach named SCARCE, i.e., \emph{Selected-Completely-At-Random ComplEmentary-label learning}, without relying on the uniform distribution assumption or an additional ordinary-label training set. Inspired by the PU learning literature, we propose the Selected-Completely-at-Random~(SCAR) assumption for complementary-label learning and propose an unbiased risk estimator accordingly. We then introduce a risk-correction approach to mitigate overfitting issues with risk consistency maintained. Furthermore, we show that complementary-label learning can be expressed as a set of negative-unlabeled~(NU) binary classification problems when using the one-versus-rest~(OVR) strategy. Table~\ref{comparison} shows the comparison between SCARCE and previous methods. The main contributions of this work are summarized as follows: 
\begin{itemize}[leftmargin=1em, itemsep=1pt, topsep=1pt, parsep=-1pt]
\item Methodologically, we propose the first consistent complementary-label learning approach without relying on the uniform distribution assumption or an additional ordinary-label dataset in non-uniform cases.  
\item Theoretically, we uncover the relation between complementary-label learning and NU learning, which provides a new perspective for understanding complementary-label learning. We also prove the convergence rate of the proposed risk estimator by providing an estimation error bound.
\item Empirically, the proposed approach is shown to achieve superior performance over state-of-the-art methods on both synthetic and real-world benchmark datasets.
\end{itemize}
\section{Preliminaries}
In this section, we review the background of learning with ordinary labels, complementary labels, and PU learning. Then, we introduce a new data distribution assumption for generating complementary labels.

\subsection{Learning with Ordinary Labels}
Let $\mathcal{X} = \mathbb{R}^d$ denote the $d$-dimensional feature space and $\mathcal{Y}=\left\{1,2,\ldots,q\right\}$ denote the label space with $q$ class labels. Let $p(\bm{x},y)$ be the joint probability density over the random variables $(\bm{x},y)\in \mathcal{X}\times \mathcal{Y}$, then the classification risk is 
\begin{equation}\label{ordinary_risk}
R(f)= \mathbb{E}_{p(\bm{x},y)}\left[\mathcal{L}(f(\bm{x}), y)\right],
\end{equation}
where $f(\bm{x})$ is the model prediction and $\mathcal{L}$ can be any \emph{classification-calibrated} loss function, such as the cross-entropy loss~\citep{bartlett2006convexity}. Let $p(\bm{x})$ denote the marginal density of unlabeled data. Besides, let $\pi_{k}=p(y=k)$ be the class-prior probability of the $k$-th class and $p(\bm{x}|y=k)$ denote the class-conditional density. Then, the classification risk in Eq.~(\ref{ordinary_risk}) can be written as
\begin{equation}
R(f)=\sum_{k=1}^{q}\left(\pi_{k}\mathbb{E}_{p(\bm{x}|y=k)}\left[\mathcal{L}(f(\bm{x}), k)\right]\right).
\end{equation}
\subsection{Learning with Complementary Labels}
In complementary-label learning, each training example is associated with one or multiple complementary labels specifying the classes to which the example does not belong. Let $\mathcal{D}=\left\{\left(\bm{x}_i, \bar{Y}_i\right)\right\}_{i=1}^{n}$ denote the complementary-label training set sampled i.i.d.~from an unknown density $p(\bm{x}, \bar{Y})$. Here, $\bm{x} \in \mathcal{X}$ is a feature vector, and $\bar{Y} \subseteq \mathcal{Y}$ is a complementary-label set associated with $\bm{x}$. In the literature, complementary-label learning can be categorized into single complementary-label learning when we have $|\bar{Y}|=1$~\citep{ishida2017learning,
gao2021discriminative,liu2023consistent}, and multiple complementary-label learning when we have $1\leq|\bar{Y}|\leq q-1$~\citep{feng2020learning}. 

In this paper, we consider a more general case where $\bar{Y}$ can contain \emph{any number} of complementary labels, ranging from zero to $q-1$. This means that our method can cover a wider range of applications and can handle additional unlabeled data without complementary labels. For ease of notation, we use a $q$-dimensional label vector $\bar{\bm{y}}=\left[\bar{y}_{1},\bar{y}_{2},\ldots,\bar{y}_{q}\right] \in \{0,1\}^q$ to denote the vector version of $\bar{Y}$, where $\bar{y}_{k}=1$ when $k \in \bar{Y}$ and $\bar{y}_{k}=0$ otherwise. Let $\bar{\pi}_{k}=p\left(\bar{y}_{k}=1\right)$ denote the fraction of training data where the $k$-th class is considered as a complementary label. Let $p\left(\bm{x}|\bar{y}_{k}=1\right)$ and $p\left(\bm{x}|\bar{y}_k=0\right)$ denote the marginal densities where the $k$-th class is considered as a complementary label or not. The task of complementary-label learning is to learn a multi-class classifier $f:\mathcal{X}\rightarrow \mathcal{Y}$ from $\mathcal{D}$. 
\subsection{Learning from Positive and Unlabeled Data}
In PU learning~\citep{elkan2008learning,du2014analysis,kiryo2017positive}, the goal is to learn a binary classifier only from a positive dataset $\mathcal{D}^{\mathrm P}=\left\{(\bm{x}_{i},+1)\right\}_{i=1}^{n^{\mathrm P}}$ and an unlabeled dataset $\mathcal{D}^{\mathrm U}=\{\bm{x}_{i}\}_{i=1}^{n^{\mathrm U}}$. There are mainly two problem settings for PU learning, i.e., the two-sample setting~\citep{du2014analysis,niu2016theoretical,chen2020variational} and the one-sample setting~\citep{elkan2008learning,coudray2023risk}. In the two-sample setting, we assume that $\mathcal{D}^{\mathrm P}$ is sampled from the positive-class density $p(\bm{x}|y=+1)$ and $\mathcal{D}^{\mathrm U}$ is sampled from the marginal density $p(\bm{x})$. In contrast, in the one-sample setting, we assume that an unlabeled dataset is first sampled from the marginal density $p(\bm{x})$. Then, if a training example is positive, its label is observed with a \emph{constant probability} $c$, and the example remains unlabeled with probability $1-c$. If a training example is negative, its label is never observed and the example remains unlabeled with probability 1. In this paper, we make use of the one-sample setting for complementary-label learning.
\subsection{Generation Process of Complementary Labels}
Inspired by the SCAR assumption in PU learning~\citep{elkan2008learning,coudray2023risk}, we introduce the SCAR assumption for generating complementary labels, which can be summarized as follows.
\begin{assumption}[Selected-Completely-at-Random~(SCAR) Assumption]\label{scar}
The complementary-label data with the $k$-th class as a complementary label are sampled completely at random from the marginal density of the data not belonging to the $k$-th class, i.e.,
\begin{align}
p\left(k\in\bar{Y}|\bm{x},k\in{\mathcal{Y}\backslash \{y\}}\right)=p\left(k\in\bar{Y}|k\in{\mathcal{Y}\backslash \{y\}}\right)=c_{k},
\end{align}
where $c_{k}=\bar{\pi}_{k}/(1-\pi_k)$ is a constant specifying the fraction of data with the $k$-th class as a complementary label and $(\bm{x},y)$ is sampled from the density $p(\bm{x},y)$.
\end{assumption}

Our motivation is that complementary labels are often generated in a \emph{class-wise} manner. They can be collected by answering ``yes'' or ``no'' questions given a pair of an example and a candidate label~\citep{hu2019active,wang2021learning}. During an annotation round, we randomly select a candidate label and ask the annotators whether the example belongs to that class or not. The process is repeated iteratively, so that each example may be annotated with \emph{multiple} complementary labels. The SCAR assumption differs from the biased distribution assumption, where only one \emph{single} complementary label is generated by sampling only once from a multinomial distribution. Moreover, the SCAR assumption can be generalized to non-uniform cases by setting $c_k$ to different values for different labels. Therefore, our assumption is more practical in real-world scenarios. 

We generate the complementary-label training set $\mathcal{D}$ as follows. First, an unlabeled dataset is sampled from $p(\bm{x})$. Then, if the latent ground-truth label of an example is not the $k$-th class, we assign it a complementary label $k$ with probability $c_k$ and still consider it to be an unlabeled example with probability $1-c_k$. We generate complementary labels for all the examples by following the procedure w.r.t.~each of the $q$ labels. The data generation process is summarized in Appendix~\ref{data_generation_process}. 
   
\section{Methodology}
In this section, we first introduce a risk rewrite formulation for complementary-label learning. Then, we propose an unbiased risk estimator, followed by its theoretical analysis. Finally, we present a risk-correction approach to improve the generalization performance.
\subsection{Risk Rewrite}
Under the SCAR assumption, the ordinary multi-class classification risk in Eq.~(\ref{ordinary_risk}) can be rewritten as follows~(the proof is given in Appendix~\ref{proof_ure_cce}).
\begin{theorem}\label{ure_cce}
Under Assumption~\ref{scar}, the classification risk in Eq.~(\ref{ordinary_risk}) can be equivalently expressed as
\begin{align}\label{ure_cce_eq}
R(f) =& \sum_{k=1}^{q}\left(\mathbb{E}_{p\left(\bm{x}|\bar{y}_{k}=1\right)}\left[\left(\bar{\pi}_{k}+\pi_{k}-1\right)\mathcal{L}(f(\bm{x}), k)\right]\right. \nonumber \\
&\left.+\mathbb{E}_{p\left(\bm{x}|\bar{y}_{k}=0\right)}\left[\left(1-\bar{\pi}_{k}\right)\mathcal{L}(f(\bm{x}), k)\right]\right).
\end{align}
\end{theorem}
Theorem~\ref{ure_cce} shows that the ordinary classification risk in Eq.~(\ref{ordinary_risk}) can be equivalently expressed using densities $p\left(\bm{x}|\bar{y}_{k}=1\right)$ and $p\left(\bm{x}|\bar{y}_{k}=0\right)$. Therefore, we can perform \emph{empirical risk minimization} by minimizing an unbiased estimation of Eq.~(\ref{ure_cce_eq}) with training data sampled from $p\left(\bm{x}|\bar{y}_{k}=1\right)$ and $p\left(\bm{x}|\bar{y}_{k}=0\right)$. 

Ideally, \emph{any multi-class loss function} can be used to instantiate $\mathcal{L}$, such as the cross-entropy loss. In addition, \emph{any model and optimizer} can be used, which reveals the universality of our proposed approach. According to our experimental results, we find that the cross-entropy loss is not robust and often leads to inferior performance, possibly due to its unboundedness~\citep{ghosh2017robust,zhang2018generalized,feng2020learning,wei2023mitigating}, which will be discussed in Section~\ref{exp_further_res}. In Section~\ref{ovr_section}, we provide an instantiation based on the OVR strategy.  
\subsection{OVR Strategy}\label{ovr_section}
The OVR strategy decomposes multi-class classification into a series of binary classification problems, which is a common strategy with extensive theoretical guarantees and sound performance~\citep{rifkin2004defense,zhang2004statistical}. It instantiates the loss function $\mathcal{L}$ in Eq.~(\ref{ordinary_risk}) with the OVR loss, i.e. 
\begin{align}\label{ordinary_risk_ovr}
R\left(f_{1}, f_{2}, \ldots, f_{q}\right) =& \mathbb{E}_{p(\bm{x},y)}\left[\ell\left(f_{y}\left(\bm{x}\right)\right)\right. \nonumber \\
&+\sum_{k\in{\mathcal{Y}\backslash \{y\}}}\ell\left(-f_{k}\left(\bm{x}\right)\right)].
\end{align}
Here, $f_{k}$ is a binary classifier w.r.t.~the $k$-th class, $\mathbb{E}$ denotes the expectation, and $\ell:\mathbb{R}\rightarrow \mathbb{R}_{+}$ is a non-negative binary-class loss function.
Then, the predicted label for a test instance $\bm{x}$ is determined as 
\begin{equation}    
f(\bm{x}) = \mathop{\arg\max}_{k\in\mathcal{Y}}~f_{k}(\bm{x}).
\end{equation}
The goal is to find optimal classifiers $f_{1}^{*},f_{2}^{*},\ldots,f_{q}^{*}$ in a function class $\mathcal{F}$ which achieve the minimum classification risk in Eq.~(\ref{ordinary_risk_ovr}), i.e., 
\begin{equation}
\left(f_{1}^{*},f_{2}^{*},\ldots,f_{q}^{*}\right) = \mathop{\arg\min}_{f_{1}, f_{2}, \ldots, f_{q}\in\mathcal{F}}~R\left(f_{1}, f_{2}, \ldots, f_{q}\right). 
\end{equation}
We show that the OVR risk can be rewritten using densities $p\left(\bm{x}|\bar{y}_{k}=1\right)$ and $p\left(\bm{x}|\bar{y}_{k}=0\right)$ as well.
\begin{theorem}\label{ure}
When the OVR loss is used, the classification risk in Eq.~(\ref{ordinary_risk_ovr}) can be equivalently expressed as $R(f_{1}, f_{2}, \ldots, f_{q})=\sum_{k=1}^{q}R_{k}(f_{k})$, where
\begin{align}
R_{k}(f_k) = &\mathbb{E}_{p\left(\bm{x}|\bar{y}_{k}=1\right)}\left[(1-\pi_{k})\ell\left(-f_{k}(\bm{x})\right)+\left(\bar{\pi}_{k}+\pi_{k}-1\right)\right. \nonumber \\
&\left.\ell\left(f_{k}(\bm{x})\right)\right]+\mathbb{E}_{p\left(\bm{x}|\bar{y}_{k}=0\right)}\left[\left(1-\bar{\pi}_{k}\right)\ell\left(f_{k}(\bm{x})\right)\right]. 
\end{align}
\end{theorem}
The proof is given in Appendix~\ref{proof_ure}. Since the true densities $p\left(\bm{x}|\bar{y}_{k}=1\right)$ and $p\left(\bm{x}|\bar{y}_{k}=0\right)$ are not directly accessible, we approximate the risk \emph{empirically}. Suppose we have binary-class datasets $\mathcal{D}^{\rm N}_{k}$ and $\mathcal{D}^{\rm U}_{k}$ sampled i.i.d.~from $p\left(\bm{x}|\bar{y}_{k}=1\right)$ and $p\left(\bm{x}|\bar{y}_{k}=0\right)$, respectively. Then, an unbiased risk estimator can be derived from these binary-class datasets to approximate the classification risk in Theorem~\ref{ure} as $\widehat{R}(f_{1}, f_{2}, \ldots, f_{q})=\sum_{k=1}^{q}\widehat{R}_{k}(f_k)$, where 
\begin{align}\label{ure_eq}
\widehat{R}_{k}(f_k)=&\frac{1}{n^{\rm N}_{k}}\sum_{i=1}^{n^{\rm N}_{k}}\left(\left(1-\pi_{k}\right)\ell\left(-f_{k}(\bm{x}_{k,i}^{\mathrm{N}})\right)+\left(\bar{\pi}_{k}+\pi_{k}-1\right)\right.\nonumber \\
&\left.\ell\left(f_{k}(\bm{x}_{k,i}^{\mathrm{N}})\right)\right) +\frac{(1-\bar{\pi}_{k})}{n^{\rm U}_{k}}\sum_{i=1}^{n^{\rm U}_{k}}\ell\left(f_{k}(\bm{x}_{k,i}^{\mathrm{U}})\right).
\end{align}
We may add regularization terms to $\widehat{R}\left(f_{1}, f_{2}, \ldots, f_{q}\right)$ when necessary~\citep{loshchilov2019decoupled}. This paper considers generating the binary-class datasets $\mathcal{D}^{\rm N}_{k}$ and $\mathcal{D}^{\rm U}_{k}$ by \emph{duplicating} instances of $\mathcal{D}$. Specifically, if the $k$-th class is a complementary label of a training example, we regard its duplicated instance as a \emph{negative example} sampled from $p\left(\bm{x}|\bar{y}_{k}=1\right)$ and put the duplicated instance
in $\mathcal{D}^{\rm N}_{k}$. If the $k$-th class is not a complementary label of a training example, we regard its duplicated instance as an \emph{unlabeled example} sampled from $p\left(\bm{x}|\bar{y}_{k}=0\right)$ and put the duplicated instance
in $\mathcal{D}^{\rm U}_{k}$. In this way, we can obtain $q$ negative binary-class datasets and $q$ unlabeled binary-class datasets~($k\in \mathcal{Y}$):
\begin{align}
\mathcal{D}^{\rm N}_{k}&=\left\{(\bm{x}_{k,i}^{\mathrm{N}}, -1)\right\}_{i=1}^{n^{\rm N}_{k}}=\left\{(\bm{x}_{j}, -1)|(\bm{x}_{j},\bar{Y}_{j})\in\mathcal{D}, k\in\bar{Y}_{j}\right\};\label{neg_binary} \\
\mathcal{D}^{\rm U}_{k}&=\left\{\bm{x}_{k,i}^{\mathrm{U}}\right\}_{i=1}^{n^{\rm U}_{k}}=\left\{\bm{x}_{j}|(\bm{x}_{j},\bar{Y}_{j})\in\mathcal{D},k\notin\bar{Y}_{j}\right\}. \label{unlabel_binary}
\end{align}
The details of our algorithm are summarized in Algorithm~\ref{scarce_algo_code}. When the class priors $\pi_{k}$ are not accessible to the learning algorithm, they can be estimated by off-the-shelf mixture proportion estimation approaches~\citep{scott2015rate,ramaswamy2016mixture,zhang2020unbiased,garg2021mixture,yao2022rethinking} with $\mathcal{D}^{\rm N}_{k}$ and $\mathcal{D}^{\rm U}_{k}$. Notably, the \emph{irreducibility}~\citep{blanchard2010semi,scott2013classification} assumption is necessary for class-prior estimation. However, it is still less demanding than the biased distribution assumption, which requires additional ordinary-label training data with deterministic labels, a.k.a.~anchor points, to estimate the transition matrix~\citep{yu2018learning}. We present the details of a class-prior estimation algorithm in Appendix~\ref{cpe_apd}. 
\begin{algorithm}[t]
\caption{SCARCE}\label{scarce_algo_code}
\noindent {\bf Input:} Complementary-label training set $\mathcal{D}$, class priors $\pi_{k}$~($k\in\mathcal{Y}$), unseen instance $\bm{x}_{*}$, epoch $T_{\text{max}}$, iteration $I_{\text{max}}$.  
\begin{algorithmic}
\FOR{$t = 1,2,\ldots,T_{\text{max}}$}
    \STATE \textbf{Shuffle} the complementary-label training set $\mathcal{D}$; 
    \FOR{$j = 1,\ldots,I_{\text{max}}$}
        \STATE \textbf{Fetch} mini-batch $\mathcal{D}_j$ from $\mathcal{D}$;
        \STATE \textbf{Update} the shared representation layers and specific classification layers $f_{1}, f_{2}, \ldots, f_{q}$ by minimizing the corrected risk estimator $\widetilde{R}\left(f_{1}, f_{2}, \ldots, f_{q}\right)$ in  Eq.~(\ref{corrected_eq});
    \ENDFOR
\ENDFOR
\STATE \textbf{Return} $y_{*} = \mathop{\arg\max}_{k\in\mathcal{Y}}~f_{k}(\bm{x}_{*})$;
\end{algorithmic}
\hspace*{0.02in} {\bf Output:} Predicted label $y_{*}$.
\end{algorithm}
\begin{figure*}[htbp]
  \centering
  \subfigure[MNIST]{
    \includegraphics[width=4cm]{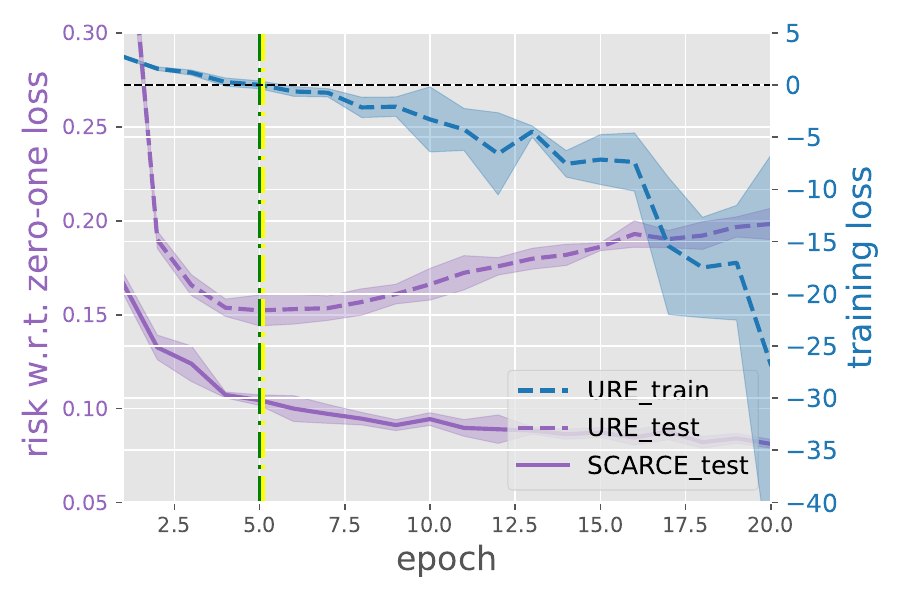}
  }
  \subfigure[Kuzushiji-MNIST]{
    \includegraphics[width=4cm]{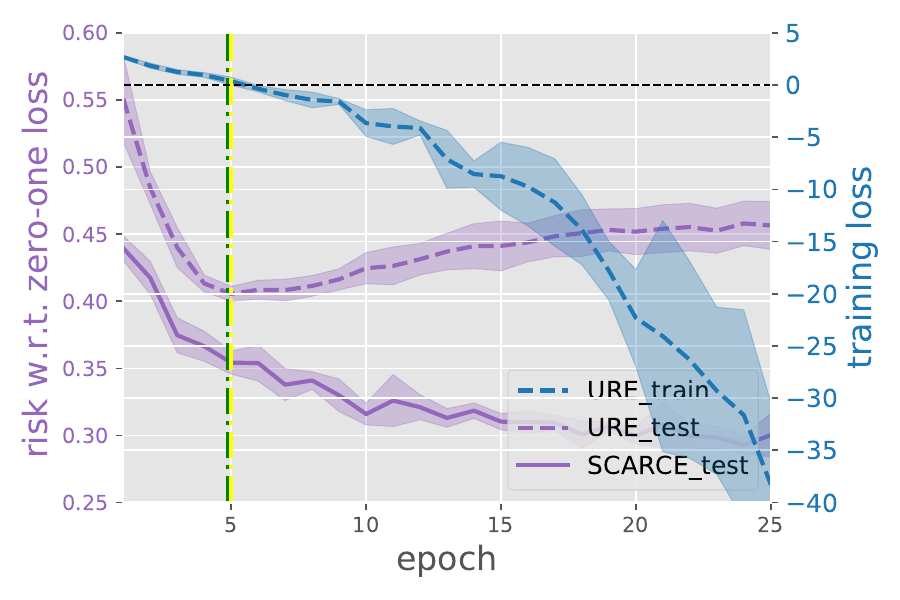}
  }
  \subfigure[Fashion-MNIST]{
    \includegraphics[width=4cm]{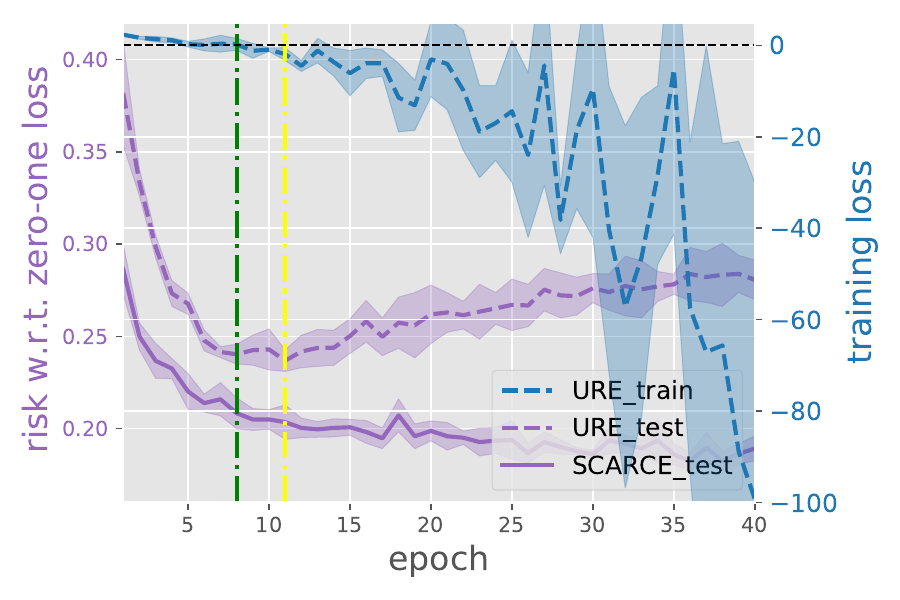}
  }
  \subfigure[CIFAR-10]{
    \includegraphics[width=4cm]{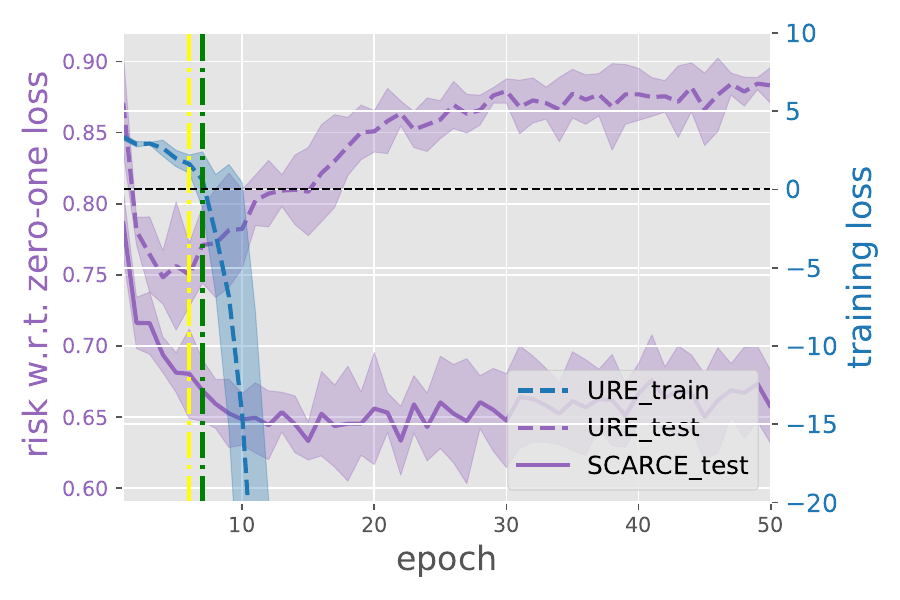}
  }
  \\
  \caption{Training curves and test curves of the method that minimizes the URE and test curves of our proposed risk-correction approach SCARCE. The green dashed lines indicate when the URE becomes negative while the yellow dashed lines indicate when the overfitting phenomena occur. The complementary labels are generated by following the uniform distribution assumption. ResNet is used as the model architecture for CIFAR-10 and MLP is used for other datasets.}\label{overfit_exp}
\end{figure*}
\subsection{Relation to Negative-Unlabeled Learning}\label{nu_analysis}
First, we provide the risk rewrite results of NU learning~\citep{niu2016theoretical}. Consider the binary classification problem with a class prior $\pi_{+}=p(y=+1)$. Suppose an NU dataset is generated by following the SCAR assumption in PU learning with the positive and negative classes swapped. Our goal is to learn a binary classifier $f^{\rm NU}$ from the NU dataset. Let $\bar{\pi}_{-}=p(\bar{y}=-1)$ denote the proportion of negative data in the entire NU dataset. Let $p(\bm{x}|\bar{y}=-1)$ and $p(\bm{x}|\bar{y}=0)$ denote the density of observed negative data and unlabeled data, respectively. 
\begin{lemma}\label{nu_risk}
Based on the assumptions above, the classification risk of binary classification can be equivalently expressed as
\begin{align}\label{nu_risk_rewrite}
&\mathbb{E}_{p(\bm{x}|\bar{y}=-1)}\left[(1-\pi_{+})\ell\left(-f^{\rm NU}(\bm{x})\right)+\left(\bar{\pi}_{-}+\pi_{+}-1\right)\right. \nonumber \\
&\left.\ell\left(f^{\rm NU}(\bm{x})\right)\right]+\mathbb{E}_{p(\bm{x}|\bar{y}=0)}\left[\left(1-\bar{\pi}_{-}\right)\ell\left(f^{\rm NU}(\bm{x})\right)\right]. 
\end{align}
\end{lemma}
We observe that the multi-class classification risk in Theorem~\ref{ure} is the sum of Eq.~(\ref{nu_risk_rewrite}) by considering each class as a positive class in turn. Furthermore, besides minimizing the NU classification risk $R_{k}(f_k)$, we can adopt any other PU learning approach~\citep{chen2020variational,garg2021mixture,li2022who,wang2023beyond,jiang2023positive,dai2023gradpu} to derive the binary classifier $f_{k}$ by swapping the positive class and the negative class. Finally, we can predict the label for a test instance as the class of the minimum model output, since the positive and negative classes are swapped.
Therefore, the proposal can be considered as a \emph{general framework} for solving complementary-label learning problems. Based on this finding, we propose a \emph{meta} complementary-label learning algorithm in Algorithm~\ref{scarce_meta_code}, and the proposed method SCARCE can be considered as an instantiation. In particular, when employing deep neural networks as the model architecture for PU learning algorithms, we can share the representation learning layers and use specific classification layers for different labels, which may allow training different classifiers simultaneously.
\begin{algorithm}[t]
\caption{SCARCE-Meta}\label{scarce_meta_code}
\noindent {\bf Input:} Complementary-label training set $\mathcal{D}$, PU learning algorithm $\mathcal{A}$, unseen instance $\bm{x}_{*}$, epoch $T_{\text{max}}$, iteration $I_{\text{max}}$, number of labels $q$.
\begin{algorithmic}
\FOR{$t = 1,2,\ldots,T_{\text{max}}$}
    \STATE \textbf{Shuffle} the complementary-label training set $\mathcal{D}$; 
    \FOR{$j = 1,\ldots,I_{\text{max}}$}
        \STATE \textbf{Fetch} a mini-batch $\mathcal{D}_j$ from $\mathcal{D}$;
        \FOR{$k = 1,\ldots,q$}
            \STATE \textbf{Construct} a negative dataset $\mathcal{D}^{\rm N}_{k}$ and an unlabeled dataset $\mathcal{D}^{\rm U}_{k}$ according to Eq.~(\ref{neg_binary}) and (\ref{unlabel_binary});
            \STATE \textbf{Train} a binary classifier $f_k \mapsfrom \mathcal{A}(\mathcal{D}^{\rm N}_{k}, \mathcal{D}^{\rm U}_{k})$ by regarding negative data as positive.
        \ENDFOR
    \ENDFOR
\ENDFOR
\STATE \textbf{Return} $y_{*} = \mathop{\arg\min}_{k\in\mathcal{Y}}~f_{k}(\bm{x}_{*})$;
\end{algorithmic}
\hspace*{0.02in} {\bf Output:} Predicted label $y_{*}$.
\end{algorithm}
\subsection{Theoretical Analysis}
\paragraph{Calibration.}We show that the proposed risk can be calibrated to the 0-1 loss \citep{zhang2004statistical}. Let $R_{\mathrm{0-1}}(f)=\mathbb{E}_{p(\bm{x},y)}\mathbb{I}(f(\bm{x})\neq y)$ denote the expected 0-1 loss where $f(\bm{x}) = \mathop{\arg\max}_{k\in\mathcal{Y}}~f_{k}(\bm{x})$ and $R_{\mathrm{0-1}}^{*}=\mathop{\min}_{f}~R_{\mathrm{0-1}}(f)$ denote the Bayes error. Besides, let $R^{*}= \mathop{\min}_{f_{1}, f_{2}, \ldots, f_{q}}~R(f_{1}, f_{2}, \ldots, f_{q})$ denote the minimum risk of the proposed risk. Then we have the following theorem~(its proof is given in Appendix~\ref{proof_calibration}).
\begin{theorem}\label{calibration}
Suppose the binary-class loss function $\ell$ is convex, bounded below, differential, and satisfies $\ell(z)\leq \ell(-z)$ when $z > 0$. Then we have that for any $\epsilon_{1} > 0$, there exists an $\epsilon_{2} > 0$ such that
\begin{equation}
R\left(f_{1}, f_{2}, \ldots, f_{q}\right) \leq R^{*} + \epsilon_{2} \Rightarrow R_{\mathrm{0-1}}(f) \leq R_{\mathrm{0-1}}^{*} + \epsilon_{1}.
\end{equation}
\end{theorem}
\begin{remark}
The infinite-sample consistency elucidates that the proposed risk can be calibrated to the 0-1 loss. Therefore, if we minimize the proposed risk and obtain the optimal classifier, the classifier also achieves the Bayes error.
\end{remark}
\paragraph{Estimation error bound.}We further elaborate the convergence property of the empirical risk estimator $\widehat{R}(f_{1}, f_{2}, \ldots, f_{q})$ by providing its estimation error bound. The optimal classifiers w.r.t.~$\widehat{R}(f_{1}, f_{2}, \ldots, f_{q})$ are 
\begin{equation}
\left(\widehat{f}_{1}, \widehat{f}_{2}, \ldots, \widehat{f}_{q}\right) = \mathop{\arg\min}_{f_{1}, f_{2}, \ldots, f_{q}\in\mathcal{F}}~\widehat{R}\left(f_{1}, f_{2}, \ldots, f_{q}\right).
\end{equation}
\begin{theorem}\label{eeb}
Based on the above assumptions, for any $\delta > 0$, the following inequality holds with probability at least $1 -\delta$:
\begin{align}
&R\left(\widehat{f}_{1}, \widehat{f}_{2}, \ldots, \widehat{f}_{q}\right) - R\left(f_{1}^{*},f_{2}^{*},\ldots,f_{q}^{*}\right)\leq \nonumber \\
&\sum_{k=1}^{q} \left(
(4-4\bar{\pi}_{k})L_{\ell}\mathfrak{R}_{n^{\rm U}_{k},p^{\rm U}_{k}}(\mathcal{F})+(1-\bar{\pi}_{k})C_{\ell}\sqrt{\frac{2\ln{\left(2/\delta\right)}}{n^{\rm U}_{k}}}\right. \nonumber \\
&+(8-8\pi_{k}-4\bar{\pi}_{k})L_{\ell}\mathfrak{R}_{n^{\rm N}_{k},p^{\rm N}_{k}}(\mathcal{F}) \nonumber \\
&\left.+(2-2\pi_{k}-\bar{\pi}_{k})C_{\ell}\sqrt{\frac{2\ln{\left(2/\delta\right)}}{n^{\rm N}_{k}}}\right),
\end{align}
where $\mathfrak{R}_{n^{\rm U}_{k},p^{\rm U}_{k}}(\mathcal{F})$ and $\mathfrak{R}_{n^{\rm N}_{k},p^{\rm N}_{k}}(\mathcal{F})$ denote the Rademacher complexity of $\mathcal{F}$ given $n^{\rm U}_{k}$ unlabeled data sampled from $p\left(\bm{x}|\bar{y}_{k}=0\right)$ and $n^{\rm N}_{k}$ negative data sampled from $p\left(\bm{x}|\bar{y}_{k}=1\right)$ respectively.
\end{theorem}

\begin{remark}
Theorem~\ref{eeb} elucidates an estimation error bound of our proposed risk estimator. When $n^{\rm U}_{k}$ and $ n^{\rm N}_{k} \rightarrow \infty$, $R\left(\widehat{f}_{1}, \widehat{f}_{2}, \ldots, \widehat{f}_{q}\right) \rightarrow R\left(f_{1}^{*},f_{2}^{*},\ldots,f_{q}^{*}\right)$ because $\mathfrak{R}_{n^{\rm U}_{k},p^{\rm U}_{k}}(\mathcal{F}) \rightarrow 0$ and $ \mathfrak{R}_{n^{\rm N}_{k},p^{\rm N}_{k}}(\mathcal{F}) \rightarrow 0$ for all parametric models with a bounded norm such as deep neural networks with weight decay~\citep{golowich2018size}. Furthermore, the estimation error bound converges in $\mathcal{O}_{p}\left(\sum_{k=1}^{q}\left(1/\sqrt{n^{\rm N}_{k}}+1/\sqrt{n^{\rm U}_{k}}\right)\right)$, where $\mathcal{O}_{p}$ denotes the order in probability.
\end{remark}
\begin{table*}[t]
\small
\caption{Classification accuracy~(mean$\pm$std) of each method on MNIST. The best performance is shown in bold~(pairwise \emph{t}-test at the 0.05 significance level).}\label{res_mnist}
\centering
\begin{tabular}{l cc cc cc cc cc} 
\toprule[1pt]	 	
Setting &\multicolumn{2}{c}{Uniform}&\multicolumn{2}{c}{Biased-a}&\multicolumn{2}{c}{Biased-b}&\multicolumn{2}{c}{SCAR-a}&\multicolumn{2}{c}{SCAR-b}\\ 
\midrule Model&  MLP & LeNet & MLP & LeNet &MLP & LeNet& MLP & LeNet & MLP & LeNet\\
\midrule	PC& \makecell{71.11 \\ $\pm$\scriptsize{0.83 }}
& \makecell{82.69 \\ $\pm$\scriptsize{1.15 }}
& \makecell{69.29 \\ $\pm$\scriptsize{0.97 }}
& \makecell{87.82 \\ $\pm$\scriptsize{0.69 }}
& \makecell{71.59 \\ $\pm$\scriptsize{0.85 }}
& \makecell{87.66 \\ $\pm$\scriptsize{0.66 }}
& \makecell{66.97 \\ $\pm$\scriptsize{1.03 }}
&\makecell{11.00 \\ $\pm$\scriptsize{0.79 }}
& \makecell{57.67 \\ $\pm$\scriptsize{0.98 }}
& \makecell{49.17 \\ $\pm$\scriptsize{35.9 }} \\ 
\midrule NN& \makecell{67.75 \\ $\pm$\scriptsize{0.96 }}
& \makecell{86.16 \\ $\pm$\scriptsize{0.69 }}
& \makecell{30.59 \\ $\pm$\scriptsize{2.31 }}
& \makecell{46.27 \\ $\pm$\scriptsize{2.61 }}
& \makecell{38.50 \\ $\pm$\scriptsize{3.93 }}
& \makecell{63.67 \\ $\pm$\scriptsize{3.75 }}
& \makecell{67.39 \\ $\pm$\scriptsize{0.68 }}
&\makecell{86.58 \\ $\pm$\scriptsize{0.95 }}
& \makecell{63.95 \\ $\pm$\scriptsize{0.56 }}
& \makecell{79.94 \\ $\pm$\scriptsize{0.48 }} \\ 
\midrule GA& \makecell{88.00 \\ $\pm$\scriptsize{0.85 }}
& \makecell{96.02 \\ $\pm$\scriptsize{0.15 }}
& \makecell{65.97 \\ $\pm$\scriptsize{7.87 }}
& \makecell{94.55 \\ $\pm$\scriptsize{0.43 }}
& \makecell{75.77 \\ $\pm$\scriptsize{1.48 }}
& \makecell{94.87 \\ $\pm$\scriptsize{0.28 }}
& \makecell{62.62 \\ $\pm$\scriptsize{2.29 }}
&\makecell{90.23 \\ $\pm$\scriptsize{0.92 }}
& \makecell{56.91 \\ $\pm$\scriptsize{2.08 }}
& \makecell{78.66 \\ $\pm$\scriptsize{0.61 }} \\ 
\midrule L-UW& \makecell{73.49 \\ $\pm$\scriptsize{0.88 }}
& \makecell{77.74 \\ $\pm$\scriptsize{0.97 }}
& \makecell{39.63 \\ $\pm$\scriptsize{0.57 }}
& \makecell{32.21 \\ $\pm$\scriptsize{1.20 }}
& \makecell{42.77 \\ $\pm$\scriptsize{1.42 }}
& \makecell{34.57 \\ $\pm$\scriptsize{1.90 }}
& \makecell{35.08 \\ $\pm$\scriptsize{1.59 }}
&\makecell{33.82 \\ $\pm$\scriptsize{2.44 }}
& \makecell{30.24 \\ $\pm$\scriptsize{1.81 }}
& \makecell{24.28 \\ $\pm$\scriptsize{2.74 }} \\ 
\midrule L-W& \makecell{62.24 \\ $\pm$\scriptsize{0.50 }}
& \makecell{63.04 \\ $\pm$\scriptsize{1.58 }}
& \makecell{36.90 \\ $\pm$\scriptsize{0.34 }}
& \makecell{29.25 \\ $\pm$\scriptsize{0.94 }}
& \makecell{41.55 \\ $\pm$\scriptsize{0.63 }}
& \makecell{32.98 \\ $\pm$\scriptsize{2.25 }}
& \makecell{33.53 \\ $\pm$\scriptsize{2.08 }}
&\makecell{26.02 \\ $\pm$\scriptsize{1.31 }}
& \makecell{28.99 \\ $\pm$\scriptsize{2.38 }}
& \makecell{23.69 \\ $\pm$\scriptsize{2.94 }} \\ 
\midrule OP& \makecell{78.87 \\ $\pm$\scriptsize{0.46 }}
& \makecell{88.76 \\ $\pm$\scriptsize{1.68 }}
& \makecell{73.46 \\ $\pm$\scriptsize{0.71 }}
& \makecell{85.96 \\ $\pm$\scriptsize{1.02 }}
& \makecell{74.16 \\ $\pm$\scriptsize{0.52 }}
& \makecell{87.23 \\ $\pm$\scriptsize{1.31 }}
& \makecell{76.29 \\ $\pm$\scriptsize{0.23 }}
&\makecell{86.94 \\ $\pm$\scriptsize{1.94 }}
& \makecell{68.12 \\ $\pm$\scriptsize{0.51 }}
& \makecell{71.67 \\ $\pm$\scriptsize{2.30 }} \\ 
\midrule SCARCE& \textbf{\makecell{91.27  \\ $\pm$\scriptsize{0.20  }}}
& \textbf{\makecell{ 97.00 \\ $\pm$\scriptsize{0.30  }}}
& \textbf{\makecell{88.14  \\ $\pm$\scriptsize{0.70  }}}
& \textbf{\makecell{96.14  \\ $\pm$\scriptsize{0.32  }}}
& \textbf{\makecell{89.51  \\ $\pm$\scriptsize{0.44  }}}
& \textbf{\makecell{96.62  \\ $\pm$\scriptsize{0.10  }}}
& \textbf{\makecell{90.98  \\ $\pm$\scriptsize{0.27  }}}
&\textbf{\makecell{96.72  \\ $\pm$\scriptsize{0.16  }}}
& \textbf{\makecell{81.85  \\ $\pm$\scriptsize{0.25  }}}
& \textbf{\makecell{87.05  \\ $\pm$\scriptsize{0.28  }}} \\ 
\bottomrule[1pt]
\end{tabular}
\end{table*}
\begin{table*}[t]
\vspace{-15pt}
\small
\caption{Classification accuracy~(mean$\pm$std) of each method on Kuzushiji-MNIST. The best performance is shown in bold~(pairwise \emph{t}-test at the 0.05 significance level).
}\label{res_kmnist}
\centering
\begin{tabular}{l cc cc cc cc cc} 
\toprule[1pt]   
Setting &\multicolumn{2}{c}{Uniform}&\multicolumn{2}{c}{Biased-a}&\multicolumn{2}{c}{Biased-b}&\multicolumn{2}{c}{SCAR-a}&\multicolumn{2}{c}{SCAR-b}\\ 
\midrule Model&  MLP & LeNet & MLP & LeNet &MLP & LeNet& MLP & LeNet & MLP & LeNet\\
\midrule PC& \makecell{42.93 \\ $\pm$\scriptsize{0.33 }}
& \makecell{56.79 \\ $\pm$\scriptsize{1.54 }}
& \makecell{41.60 \\ $\pm$\scriptsize{0.97 }}
& \makecell{67.39 \\ $\pm$\scriptsize{1.04 }}
& \makecell{42.53 \\ $\pm$\scriptsize{0.80 }}
& \makecell{66.81 \\ $\pm$\scriptsize{1.33 }}
& \makecell{39.58 \\ $\pm$\scriptsize{1.35 }}
&\makecell{42.59 \\ $\pm$\scriptsize{29.8}}
& \makecell{33.95 \\ $\pm$\scriptsize{1.14 }}
& \makecell{37.67 \\ $\pm$\scriptsize{25.3}} \\ 
\midrule NN& \makecell{39.42 \\ $\pm$\scriptsize{0.68 }}
& \makecell{58.57 \\ $\pm$\scriptsize{1.15 }}
& \makecell{23.97 \\ $\pm$\scriptsize{2.53 }}
& \makecell{31.10 \\ $\pm$\scriptsize{2.95 }}
& \makecell{29.93 \\ $\pm$\scriptsize{1.80 }}
& \makecell{48.72 \\ $\pm$\scriptsize{2.89 }}
& \makecell{39.31 \\ $\pm$\scriptsize{1.18 }}
&\makecell{56.84 \\ $\pm$\scriptsize{2.10 }}
& \makecell{38.68 \\ $\pm$\scriptsize{0.58 }}
& \makecell{56.70 \\ $\pm$\scriptsize{1.08 }} \\ 
\midrule GA& \makecell{60.83 \\ $\pm$\scriptsize{1.37 }}
& \makecell{76.17 \\ $\pm$\scriptsize{0.44 }}
& \makecell{43.22 \\ $\pm$\scriptsize{3.03 }}
& \textbf{\makecell{75.04 \\ $\pm$\scriptsize{0.92 }}}
& \makecell{48.03 \\ $\pm$\scriptsize{2.93 }}
& \textbf{\makecell{77.05 \\ $\pm$\scriptsize{1.67 }}}
& \makecell{36.56 \\ $\pm$\scriptsize{2.96 }}
&\makecell{59.16 \\ $\pm$\scriptsize{3.30 }}
& \makecell{33.02 \\ $\pm$\scriptsize{2.31 }}
& \makecell{52.92 \\ $\pm$\scriptsize{2.39 }} \\ 
\midrule L-UW& \makecell{43.00 \\ $\pm$\scriptsize{1.20 }}
& \makecell{49.31 \\ $\pm$\scriptsize{1.95 }}
& \makecell{27.89 \\ $\pm$\scriptsize{0.51 }}
& \makecell{25.82 \\ $\pm$\scriptsize{0.78 }}
& \makecell{31.53 \\ $\pm$\scriptsize{0.42 }}
& \makecell{30.05 \\ $\pm$\scriptsize{1.63 }}
& \makecell{21.49 \\ $\pm$\scriptsize{0.57 }}
&\makecell{19.71 \\ $\pm$\scriptsize{1.44 }}
& \makecell{18.36 \\ $\pm$\scriptsize{1.23 }}
& \makecell{16.67 \\ $\pm$\scriptsize{1.86 }} \\ 
\midrule L-W& \makecell{37.21 \\ $\pm$\scriptsize{0.59 }}
& \makecell{42.69 \\ $\pm$\scriptsize{2.54 }}
& \makecell{26.75 \\ $\pm$\scriptsize{0.61 }}
& \makecell{25.86 \\ $\pm$\scriptsize{0.64 }}
& \makecell{30.10 \\ $\pm$\scriptsize{0.57 }}
& \makecell{27.94 \\ $\pm$\scriptsize{1.68 }}
& \makecell{21.22 \\ $\pm$\scriptsize{0.77 }}
&\makecell{18.28 \\ $\pm$\scriptsize{2.11 }}
& \makecell{18.41 \\ $\pm$\scriptsize{1.66 }}
& \makecell{16.25 \\ $\pm$\scriptsize{1.51 }} \\ 
\midrule OP& \makecell{51.78 \\ $\pm$\scriptsize{0.41 }}
& \makecell{65.94 \\ $\pm$\scriptsize{1.38 }}
& \makecell{45.66 \\ $\pm$\scriptsize{0.90 }}
& \makecell{65.59 \\ $\pm$\scriptsize{1.71 }}
& \makecell{47.47 \\ $\pm$\scriptsize{1.26 }}
& \makecell{64.65 \\ $\pm$\scriptsize{1.68 }}
& \makecell{49.95 \\ $\pm$\scriptsize{0.79 }}
&\makecell{59.93 \\ $\pm$\scriptsize{1.38 }}
& \makecell{42.72 \\ $\pm$\scriptsize{0.95 }}
& \makecell{56.36 \\ $\pm$\scriptsize{2.15 }} \\ 
\midrule SCARCE& \textbf{\makecell{67.95 \\ $\pm$\scriptsize{1.29  }}}
& \textbf{\makecell{79.81  \\ $\pm$\scriptsize{1.19  }}}
& \textbf{\makecell{62.43  \\ $\pm$\scriptsize{1.02  }}}
& \textbf{\makecell{75.99  \\ $\pm$\scriptsize{0.91  }}}
& \textbf{\makecell{64.98  \\ $\pm$\scriptsize{0.72  }}}
& \textbf{\makecell{78.53  \\ $\pm$\scriptsize{0.57  }}}
& \textbf{\makecell{66.72  \\ $\pm$\scriptsize{0.69  }}}
&\textbf{\makecell{78.27  \\ $\pm$\scriptsize{1.09  }}}
& \textbf{\makecell{61.78  \\ $\pm$\scriptsize{0.36  }}}
& \textbf{\makecell{72.03  \\ $\pm$\scriptsize{0.45  }}} \\ 
\bottomrule[1pt]   
\end{tabular}
\end{table*}
\begin{table*}[htbp]
\small
\caption{Classification accuracy~(mean$\pm$std) of each method on Fashion-MNIST. The best performance is shown in bold~(pairwise \emph{t}-test at the 0.05 significance level).
}\label{res_fashion}
\centering
\begin{tabular}{l cc cc cc cc cc} 
\toprule[1pt]   
Setting &\multicolumn{2}{c}{Uniform}&\multicolumn{2}{c}{Biased-a}&\multicolumn{2}{c}{Biased-b}&\multicolumn{2}{c}{SCAR-a}&\multicolumn{2}{c}{SCAR-b}\\ 
\midrule Model&  MLP & LeNet & MLP & LeNet &MLP & LeNet& MLP & LeNet & MLP & LeNet\\
\midrule PC& \makecell{64.82 \\ $\pm$\scriptsize{1.27 }}
& \makecell{69.56 \\ $\pm$\scriptsize{1.82 }}
& \makecell{61.14 \\ $\pm$\scriptsize{1.09 }}
& \makecell{72.89 \\ $\pm$\scriptsize{1.26 }}
& \makecell{61.20 \\ $\pm$\scriptsize{0.79 }}
& \makecell{73.04 \\ $\pm$\scriptsize{1.38 }}
& \makecell{63.08 \\ $\pm$\scriptsize{0.88 }}
&\makecell{23.28 \\ $\pm$\scriptsize{29.7 }}
& \makecell{47.23 \\ $\pm$\scriptsize{2.38 }}
& \makecell{37.53 \\ $\pm$\scriptsize{25.2}} \\ 
\midrule NN& \makecell{63.89 \\ $\pm$\scriptsize{0.92 }}
& \makecell{70.34 \\ $\pm$\scriptsize{1.09 }}
& \makecell{25.66 \\ $\pm$\scriptsize{2.12 }}
& \makecell{36.93 \\ $\pm$\scriptsize{3.86 }}
& \makecell{30.75 \\ $\pm$\scriptsize{0.96 }}
& \makecell{40.88 \\ $\pm$\scriptsize{3.71 }}
& \makecell{63.47 \\ $\pm$\scriptsize{0.70 }}
&\makecell{70.83 \\ $\pm$\scriptsize{0.87 }}
& \makecell{55.96 \\ $\pm$\scriptsize{1.55 }}
& \makecell{63.06 \\ $\pm$\scriptsize{1.38 }} \\ 
\midrule GA& \makecell{77.04 \\ $\pm$\scriptsize{0.95 }}
& \makecell{81.91 \\ $\pm$\scriptsize{0.43 }}
& \makecell{50.04 \\ $\pm$\scriptsize{4.30 }}
& \makecell{74.73 \\ $\pm$\scriptsize{0.96 }}
& \makecell{49.02 \\ $\pm$\scriptsize{5.76 }}
& \makecell{75.66 \\ $\pm$\scriptsize{1.10 }}
& \makecell{54.74 \\ $\pm$\scriptsize{3.04 }}
&\makecell{74.75 \\ $\pm$\scriptsize{1.17 }}
& \makecell{44.75 \\ $\pm$\scriptsize{3.04 }}
& \makecell{60.01 \\ $\pm$\scriptsize{1.47 }} \\ 
\midrule L-UW& \textbf{\makecell{80.29 \\ $\pm$\scriptsize{0.44 }}}
& \makecell{72.43 \\ $\pm$\scriptsize{2.07 }}
& \makecell{40.26 \\ $\pm$\scriptsize{2.49 }}
& \makecell{29.46 \\ $\pm$\scriptsize{1.70 }}
& \makecell{43.55 \\ $\pm$\scriptsize{1.61 }}
& \makecell{33.53 \\ $\pm$\scriptsize{1.35 }}
& \makecell{35.71 \\ $\pm$\scriptsize{1.50 }}
&\makecell{30.73 \\ $\pm$\scriptsize{1.64 }}
& \makecell{31.43 \\ $\pm$\scriptsize{2.98 }}
& \makecell{22.03 \\ $\pm$\scriptsize{3.62 }} \\ 
\midrule L-W& \makecell{75.14 \\ $\pm$\scriptsize{0.40 }}
& \makecell{61.89 \\ $\pm$\scriptsize{0.88 }}
& \makecell{39.87 \\ $\pm$\scriptsize{0.95 }}
& \makecell{27.57 \\ $\pm$\scriptsize{1.70 }}
& \makecell{42.02 \\ $\pm$\scriptsize{1.41 }}
& \makecell{32.69 \\ $\pm$\scriptsize{0.68 }}
& \makecell{31.86 \\ $\pm$\scriptsize{2.16 }}
&\makecell{27.37 \\ $\pm$\scriptsize{2.30 }}
& \makecell{30.26 \\ $\pm$\scriptsize{1.68 }}
& \makecell{21.61 \\ $\pm$\scriptsize{2.12 }} \\ 
\midrule OP& \makecell{69.03 \\ $\pm$\scriptsize{0.71 }}
& \makecell{71.28 \\ $\pm$\scriptsize{0.94 }}
& \makecell{62.93 \\ $\pm$\scriptsize{1.25 }}
& \makecell{70.82 \\ $\pm$\scriptsize{1.15 }}
& \makecell{62.25 \\ $\pm$\scriptsize{0.36 }}
& \makecell{68.94 \\ $\pm$\scriptsize{2.78 }}
& \makecell{66.29 \\ $\pm$\scriptsize{0.60 }}
&\makecell{69.52 \\ $\pm$\scriptsize{1.18 }}
& \makecell{56.55 \\ $\pm$\scriptsize{1.39 }}
& \makecell{56.39 \\ $\pm$\scriptsize{3.03 }} \\ 
\midrule SCARCE& \textbf{\makecell{80.44  \\ $\pm$\scriptsize{0.19  }}}
& \textbf{\makecell{82.74  \\ $\pm$\scriptsize{0.39  }}}
& \textbf{\makecell{70.08  \\ $\pm$\scriptsize{2.53  }}}
& \textbf{\makecell{79.74  \\ $\pm$\scriptsize{1.10  }}}
& \textbf{\makecell{71.97  \\ $\pm$\scriptsize{1.09  }}}
& \textbf{\makecell{80.43  \\ $\pm$\scriptsize{0.69  }}}
& \textbf{\makecell{79.75  \\ $\pm$\scriptsize{0.60  }}}
&\textbf{\makecell{82.55  \\ $\pm$\scriptsize{0.30  }}}
& \textbf{\makecell{71.16  \\ $\pm$\scriptsize{0.66  }}}
& \textbf{\makecell{ 72.79 \\ $\pm$\scriptsize{0.62  }}} \\ 
\bottomrule[1pt]   
\end{tabular}
\end{table*}
\begin{table*}[t]
\vspace{-15pt}
\footnotesize
\caption{Classification accuracy~(mean$\pm$std) of each method on CLCIFAR-10 and CLCIFAR-20. The best performance is shown in bold~(pairwise \emph{t}-test at the 0.05 significance level).
}\label{real_res}
\setlength{\tabcolsep}{3pt}
\centering
\begin{tabular}{llcccccccccc} 
\toprule[1pt]	 	
\multicolumn{1}{c}{Dataset} & \multicolumn{1}{c}{Model} & CC & PRODEN & EXP & MAE & Phuber-CE & LWS & CAVL & IDGP & POP & SCARCE \\ 
\midrule \multirow{3.3}{*}{CLCIFAR-10} & ResNet &  \makecell{31.56 \\ $\pm$\scriptsize{2.17 }}
& \makecell{26.37 \\ $\pm$\scriptsize{0.98 }}
& \makecell{34.84 \\ $\pm$\scriptsize{4.19 }}
& \makecell{19.48 \\ $\pm$\scriptsize{2.88 }}
& \textbf{\makecell{41.13 \\ $\pm$\scriptsize{0.74 }}}
& \makecell{13.05 \\ $\pm$\scriptsize{4.18 }}
& \makecell{24.12 \\ $\pm$\scriptsize{3.32 }}
&\makecell{10.00 \\ $\pm$\scriptsize{0.00 }}
&\makecell{26.75  \\ $\pm$\scriptsize{1.28  }}
& \textbf{\makecell{42.04 \\ $\pm$\scriptsize{0.96 }}} \\ \cmidrule{2-12} & DenseNet &  \makecell{37.03 \\ $\pm$\scriptsize{1.77 }}
& \makecell{31.31 \\ $\pm$\scriptsize{1.06 }}
& \textbf{\makecell{43.27 \\ $\pm$\scriptsize{1.33 }}}
& \makecell{22.77 \\ $\pm$\scriptsize{0.22 }}
& \makecell{39.92 \\ $\pm$\scriptsize{0.91 }}
& \makecell{10.00 \\ $\pm$\scriptsize{0.00 }}
& \makecell{25.31 \\ $\pm$\scriptsize{4.06 }}
&\makecell{10.00 \\ $\pm$\scriptsize{0.00 }}
&\makecell{31.45  \\ $\pm$\scriptsize{1.16  }}
& \textbf{\makecell{44.41 \\ $\pm$\scriptsize{0.43 }}} \\ \midrule \multirow{3.3}{*}{CLCIFAR-20} & ResNet &  \makecell{5.00 \\ $\pm$\scriptsize{0.00 }}
& \makecell{6.69 \\ $\pm$\scriptsize{0.31 }}
& \makecell{7.21 \\ $\pm$\scriptsize{0.17 }}
& \makecell{5.00 \\ $\pm$\scriptsize{0.00 }}
& \makecell{8.10 \\ $\pm$\scriptsize{0.18 }}
& \makecell{5.20 \\ $\pm$\scriptsize{0.45 }}
& \makecell{5.00 \\ $\pm$\scriptsize{0.00 }}
&\makecell{4.96 \\ $\pm$\scriptsize{0.09 }}
&\makecell{6.40  \\ $\pm$\scriptsize{0.33  }}
& \textbf{\makecell{20.08 \\ $\pm$\scriptsize{0.62 }}} \\ \cmidrule{2-12} & DenseNet &  \makecell{5.00 \\ $\pm$\scriptsize{0.00 }}
& \makecell{5.00 \\ $\pm$\scriptsize{0.00 }}
& \makecell{7.51 \\ $\pm$\scriptsize{0.91 }}
& \makecell{5.67 \\ $\pm$\scriptsize{1.49 }}
& \makecell{7.22 \\ $\pm$\scriptsize{0.39 }}
& \makecell{5.00 \\ $\pm$\scriptsize{0.00 }}
& \makecell{5.09 \\ $\pm$\scriptsize{0.13 }}
&\makecell{5.00 \\ $\pm$\scriptsize{0.00 }}
&\makecell{5.00  \\ $\pm$\scriptsize{0.00  }}
& \textbf{\makecell{19.91 \\ $\pm$\scriptsize{0.68 }}} \\
\bottomrule[1pt]
\end{tabular}
\end{table*}
\subsection{Risk-Correction Approach}
Although the unbiased risk estimator~(URE) has sound theoretical properties, we have found that it can encounter several overfitting problems when using complex models such as deep neural networks. The training curves and test curves of the method that works by minimizing the URE in Eq.~(\ref{ure_eq}) are shown in Figure~\ref{overfit_exp}. We can observe that the overfitting phenomena often occur almost simultaneously when the training loss becomes negative. We conjecture the overfitting problems are related with the negative terms in Eq.~(\ref{ure_eq})~\citep{kiryo2017positive,lu2020mitigating,cao2021learning}. Therefore, following~\citet{lu2020mitigating,wang2023binary}, we wrap each potentially negative term with a \emph{non-negative risk-correction function} $g(z)$, such as the absolute value function $g(z)=|z|$. For ease of notation, we introduce 
\begin{align}
\widehat{R}^{\rm P}_{k}(f_{k})=&\frac{\bar{\pi}_{k}+\pi_{k}-1}{n^{\rm N}_{k}}\sum_{i=1}^{n^{\rm N}_{k}}\ell\left(f_{k}(\bm{x}_{k,i}^{\mathrm{N}})\right) \nonumber \\
&+\frac{1-\bar{\pi}_{k}}{n^{\rm U}_{k}}\sum_{i=1}^{n^{\rm U}_{k}}\ell\left(f_{k}(\bm{x}_{k,i}^{\mathrm{U}})\right).
\end{align}
Then, the corrected risk estimator can be written as $\widetilde{R}\left(f_{1}, f_{2}, \ldots, f_{q}\right)=\sum_{k=1}^{q}\widetilde{R}_{k}(f_{k})$, where
\begin{equation} \label{corrected_eq}
\widetilde{R}_{k}(f_{k})=g\left(\widehat{R}^{\rm P}_{k}(f_{k})\right)+\frac{1-\pi_{k}}{n^{\rm N}_{k}}\sum_{i=1}^{n^{\rm N}_{k}}\ell\left(-f_{k}(\bm{x}_{k,i}^{\mathrm{N}})\right).
\end{equation}
It is obvious that Eq.~(\ref{corrected_eq}) is an upper bound of Eq.~(\ref{ure_eq}), so the bias is always present. Next, we perform a theoretical analysis to clarify that the corrected risk estimator is \emph{biased but consistent}. Since $\mathbb{E}\left[\widehat{R}^{\rm P}_{k}(f_{k})\right]=\pi_{k}\mathbb{E}_{p(\bm{x}|y=k)}\ell\left(f_{k}(\bm{x})\right)$ is non-negative, we assume that there exists a \emph{non-negative constant} $\beta$ such that for $\forall k\in\mathcal{Y}, \mathbb{E}\left[\widehat{R}^{\rm P}_{k}(f_{k})\right] \geq \beta$. We also assume that the risk-correction function $g(z)$ is Lipschitz continuous with a Lipschitz constant $L_{g}$. Besides, we assume that there exists some constant $C_{\mathfrak{R}}$ such that the Rademacher complexity $\mathfrak{R}_{n,p}(\mathcal{F})$ for unlabeled~(with $n=n^{\rm U}_{k},p=p^{\rm U}_{k}$) and negative data~(with $n=n^{\rm N}_{k},p=p^{\rm N}_{k}$) satisfies $\mathfrak{R}_{n,p}(\mathcal{F}) \leq C_{\mathfrak{R}}/\sqrt{n}$. This assumption holds for many models, such as fully connected neural networks and linear-in-parameter models with a bounded norm~\citep{golowich2018size,lu2020mitigating}. 
We introduce $\left(\widetilde{f}_{1}, \widetilde{f}_{2}, \ldots, \widetilde{f}_{q}\right) = \mathop{\arg\min}_{f_{1}, f_{2}, \ldots, f_{q}\in\mathcal{F}}~\widetilde{R}\left(f_{1}, f_{2}, \ldots, f_{q}\right)$ and $\Delta_{k}=\exp{\left(-2\beta^{2}/\left((1 - \pi_{k}-\bar{\pi}_{k})^{2}C_{\ell}^{2}/n^{\rm N}_{k}+(1-\bar{\pi}_{k})^{2}C_{\ell}^{2}/n^{\rm U}_{k}\right)\right)}$. Then we have the following theorems~(the proofs are given in Appendix~\ref{proof_bias_and_consistency} and \ref{proof_corrected_eeb} respectively).
\begin{theorem}\label{bias_and_consistency}
Based on the above assumptions, the bias of the expectation of the corrected risk estimator has the following lower and upper bounds: 
\begin{align}\label{bias_ineq}
0&\leq\mathbb{E}[\widetilde{R}(f_{1}, f_{2}, \ldots, f_{q})]-R(f_{1}, f_{2}, \ldots, f_{q}) \nonumber \\
&\leq \sum_{k=1}^{q}\left(2-2\bar{\pi}_{k}-\pi_{k}\right)\left(L_{g}+1\right)C_{\ell}\Delta_{k}.
\end{align}
Furthermore, for any $\delta > 0$, the following inequality holds with probability at least $1-\delta$:
\begin{align}
&|\widetilde{R}(f_{1}, f_{2}, \ldots, f_{q})-R(f_{1}, f_{2}, \ldots, f_{q})| \nonumber \\
&\leq\mathcal{O}_{p}\left(\sum_{k=1}^{q}\left(1/\sqrt{n^{\rm N}_{k}}+1/\sqrt{n^{\rm U}_{k}}\right)\right).
\end{align}
\end{theorem}
\begin{theorem}\label{corrected_eeb}
Based on the above assumptions, for any $\delta>0$, the following inequality holds with probability at least $1-\delta$:
\begin{align}
&R(\widetilde{f}_{1}, \widetilde{f}_{2}, \ldots, \widetilde{f}_{q}) - R(f_{1}^{*},f_{2}^{*},\ldots,f_{q}^{*}) \nonumber \\
&\leq\mathcal{O}_{p}\left(\sum_{k=1}^{q}\left(1/\sqrt{n^{\rm N}_{k}}+1/\sqrt{n^{\rm U}_{k}}\right)\right). 
\end{align}
\end{theorem}
\begin{remark}
Theorem~\ref{bias_and_consistency} shows that $\widetilde{R}(f_{1}, f_{2}, \ldots, f_{q})\rightarrow R(f_{1}, f_{2}, \ldots, f_{q})$ as $n^{\rm U}_{k}$ and $ n^{\rm N}_{k} \rightarrow \infty$, indicating that the corrected risk estimator is biased but consistent. An estimation error bound is also shown in Theorem~\ref{corrected_eeb}. The convergence rate of the estimation error bound is still the same after employing the risk-correction function.
\end{remark}
\section{Experiments}

In this section, we validate the effectiveness of SCARCE through extensive experiments.\footnote{Our implementation of SCARCE is available at \url{https://github.com/wwangwitsel/SCARCE}.}
\subsection{Experiments on Synthetic Benchmark Datasets}\label{exp_results_syn_section}
We conducted experiments on synthetic benchmark datasets, including MNIST~\citep{lecun1998gradient}, Kuzushiji-MNIST~\citep{clanuwat2018deep}, Fashion-MNIST~\citep{xiao2017fashion}, and CIFAR-10~\citep{krizhevsky2009learning}. We considered various generation processes of complementary labels by following the uniform, biased, and SCAR assumptions. We evaluated the classification performance of SCARCE against six single complementary-label learning methods, including PC~\citep{ishida2017learning}, NN~\citep{ishida2019complementary}, GA~\citep{ishida2019complementary}, L-UW~\citep{gao2021discriminative}, L-W~\citep{gao2021discriminative}, and OP~\citep{liu2023consistent}. The logistic loss was adopted to instantiate the binary-class loss function $l$, and the absolute value function was used as the risk-correction function $g$ for SCARCE. All the methods were implemented in PyTorch~\citep{paszke2019pytorch}. We used the Adam optimizer~\citep{kingma2015adam}. For a fair comparison, we use the same hyperparameters for all methods. The learning rate and batch size were fixed to 1e-3 and 256 for all the datasets, respectively. The weight decay was 1e-3 for CIFAR-10 and 1e-5 for the other three datasets. The number of epochs was set to 200, and we recorded the mean accuracy in the last ten epochs. Details of the datasets, models, and method descriptions can be found in Appendix~\ref{exp_setup}.
We assumed that the class priors were accessible to the learning algorithm. We randomly generated complementary labels five times with different seeds and recorded the mean accuracy and standard deviations. In addition, a pairwise \emph{t}-test at the 0.05 significance level is performed to show whether the performance advantages are significant.  

Tables~\ref{res_mnist}, \ref{res_kmnist}, and \ref{res_fashion} show the classification performance of each method with different models and generation settings of complementary labels on MNIST, Kuzushiji-MNIST, and Fashion-MNIST respectively. The experimental result on CIFAR-10 is shown in Appendix~\ref{more_res}. We can observe that: a) Out of 40 cases of different distributions and datasets, SCARCE achieves the best performance in 39 cases, which clearly validates its effectiveness. b) Some consistent approaches based on the uniform distribution assumption can achieve comparable or better performance for the ``uniform'' setting. For example, GA outperforms SCARCE on CIFAR-10. However, its performance drops significantly on other distribution settings. 
\subsection{Experiments on Real-World Benchmark Datasets}
We also verified the effectiveness of SCARCE on two real-world complementary-label datasets CLCIFAR-10 and CLCIFAR-20~\citep{wang2023clcifar}. The datasets were annotated by human annotators from Amazon Mechanical Turk~(MTurk). The distribution of complementary labels is too complex to be captured by any of the above assumptions. Moreover, the complementary labels may be \emph{noisy}, which means that the complementary labels may be annotated as ground-truth labels by mistake. There are three human-annotated complementary labels for each example, so they can be considered as multiple complementary-label datasets. We evaluated the classification performance of SCARCE against nine multiple complementary-label learning or partial-label learning methods, including CC~\citep{feng2020provably}, PRODEN~\citep{lv2020progressive}, EXP~\citep{feng2020learning}, MAE~\citep{feng2020learning}, Phuber-CE~\citep{feng2020learning}, LWS~\citep{wen2021leveraged}, CAVL~\citep{zhang20222exploiting}, IDGP~\citep{qiao2023decompositional}, and POP~\citep{xu2023progressive}. For a fair comparison, we set the same hyperparameters for all the methods as those of CIFAR-10 in Section~\ref{exp_results_syn_section}. Details of the models and method descriptions can be found in Appendix~\ref{exp_setup}. We found that the performance of some approaches was unstable with different network initialization, so we randomly initialized the network five times with different seeds and recorded the mean accuracy and standard deviations. 

Table~\ref{real_res} shows the experimental results on CLCIFAR-10 and CLCIFAR-20 with different models. We can observe that: a) SCARCE achieves the best performance in all cases, further confirming its effectiveness. b) The superiority is even more evident on CLCIFAR-20, a more complex dataset with extremely limited supervision. It demonstrates the advantages of SCARCE in dealing with real-world datasets. 

\subsection{Further Analysis}\label{exp_further_res}

\begin{figure}[htbp]
\centering
\subfigure[Comparison between Different Instantiations of SCARCE]{
    \includegraphics[width=3.9cm]{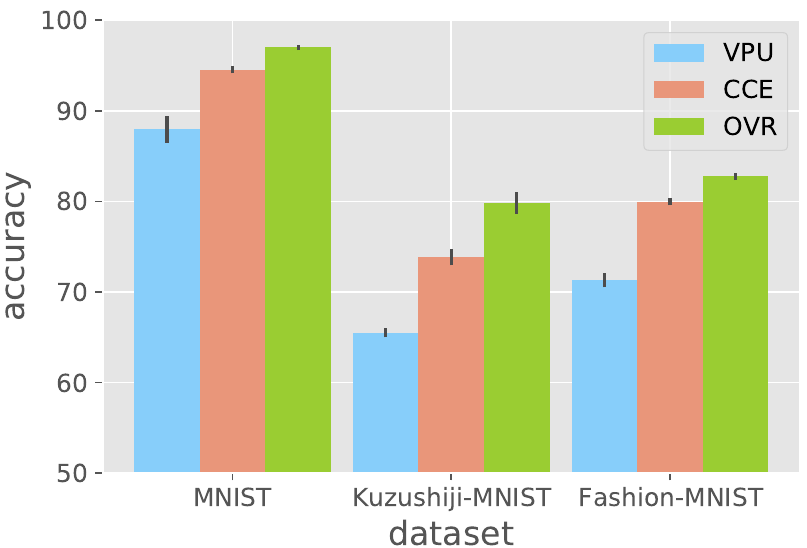}
}
\subfigure[Sensitivity Analysis]{
    \includegraphics[width=3.9cm]{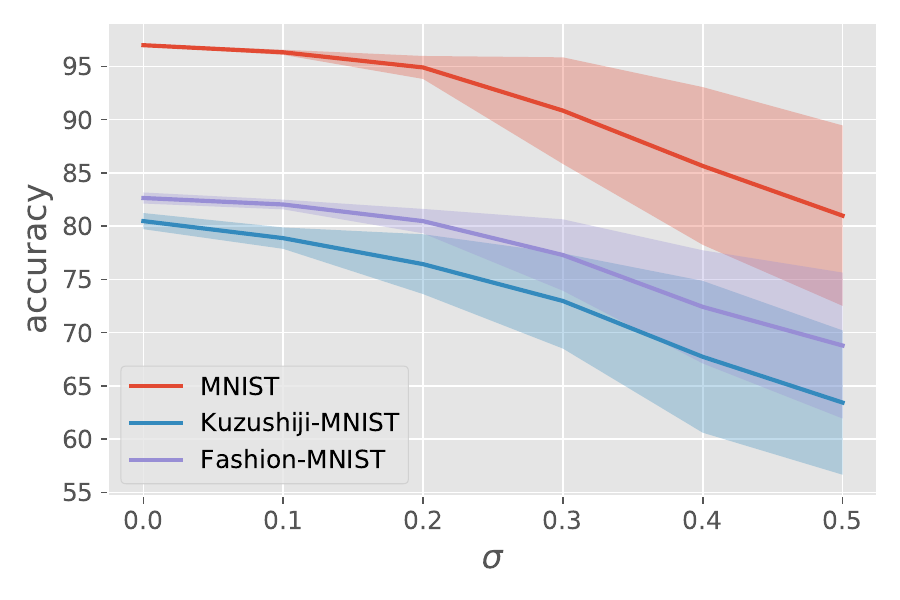}
}
\caption{(a) Classification accuracy of different instantiations of SCARCE on different datasets.~(b) Classification accuracy given inaccurate class priors.}\label{sensitivity}
\end{figure}
\paragraph{Comparison between different instantiations of SCARCE.}In Theorem~\ref{ure_cce}, any multi-class loss function can be used to instantiate $\mathcal{L}$. Therefore, we also investigated the classification performance of the cross-entropy loss~(CCE). Furthermore, we adopted the risk-correction approach to mitigate overfitting problems. We also included another instantiation of the meta-algorithm in Section~\ref{nu_analysis}. We used VPU~\citep{chen2020variational} as the PU learning approach. For a fair comparison, we did not use the mixup loss. We compared them with the default instantiation of SCARCE, i.e. the OVR loss, and Figure~\ref{sensitivity}~(a) shows the experimental results. We generated complementary labels with the uniform distribution assumption and used LeNet as the model architecture. We can observe that the OVR loss outperforms CCE and VPU. We conjecture that the inferior performance of CCE may be related to its unboundedness~\citep{ghosh2017robust,zhang2018generalized,wei2023mitigating}.
\paragraph{Sensitivity analysis.}We investigated the influence of inaccurate class priors on the classification performance of SCARCE. Specifically, let $\bar{\pi}_{k}=\xi_{k}\pi_{k}$ denote the corrupted class prior probability for the $k$-th class where $\xi_{k}$ is sampled from a normal distribution $\mathcal{N}(1, \sigma^2)$. 
We further normalized the obtained corrupted class priors to ensure that they sum up to one. 
Figure~\ref{sensitivity}~(b) shows the classification performance given inaccurate class priors using the uniform generation process and LeNet as the model architecture. From Figure~\ref{sensitivity}~(b), we can see that the performance is still satisfactory with small perturbations of the class priors. However, the performance will degenerate if the class priors deviate too much from the ground-truth values. 

\section{Conclusion}
In this paper, we proposed the first attempt towards consistent complementary-label learning without relying on the uniform distribution assumption or an ordinary-label training set to estimate the transition matrix in non-uniform cases. Based on a more practical distribution assumption, a consistent approach was proposed with theoretical guarantees. We also observed that complementary-label learning could be expressed as a set of NU classification problems when using the OVR strategy. Extensive experimental results on benchmark datasets validated the effectiveness of our proposed approach. 
\section*{Acknowledgements}
The authors thank Hsuan-Tien Lin, Deng-Bao Wang, and Yi Gao for helpful comments and discussions. The authors thank the anonymous reviewers for their helpful comments and suggestions. WW was supported by the SGU MEXT Scholarship, by the Junior Research Associate~(JRA) program of RIKEN, and by Microsoft Research Asia. TI was supported by KAKENHI Grant Number 22K17946. MS was supported by Institute for AI and Beyond, UTokyo,
and by a grant from Apple, Inc. Any views, opinions, findings, and conclusions or recommendations
expressed in this material are those of the authors and should not be
interpreted as reflecting the views, policies or position, either
expressed or implied, of Apple Inc.
\section*{Impact Statement}
This paper presents work that aims to advance the field of machine learning. There are many potential societal consequences of our work, none of which we feel need to be highlighted here.

\bibliographystyle{icml2024}
\bibliography{conu}
\newpage
\appendix
\onecolumn
\section{Class-Prior Estimation}\label{cpe_apd}
When the class priors $\pi_{k}$ are not accessible to the learning algorithm, they can be estimated by off-the-shelf mixture proportion estimation approaches~\citep{ramaswamy2016mixture,scott2015rate,garg2021mixture,yao2022rethinking}. In this section, we discuss the problem formulation and how to adapt a state-of-the-art class-prior estimation method to our problem as an example.
\paragraph{Mixture proportion estimation.} Let $F$ be a mixture distribution of two component distributions $G$ and $H$ with a proportion $\theta^{*}$, i.e., 
\begin{equation}
F=(1-\theta^{*})G+\theta^{*}H. \nonumber
\end{equation}
The task of mixture proportion estimation problems is to estimate $\theta^{*}$ given training examples sampled from $F$ and $H$. For PU learning, we consider $F=p(\bm{x})$, $G=p(\bm{x}|y=-1)$, and $H=p(\bm{x}|y=+1)$. Then, the estimation of $\theta^{*}$ corresponds to the estimation of the class prior $p(y=+1)$. It is shown that $\theta^{*}$ cannot be identified without any additional assumptions~\citep{scott2013classification, scott2015rate}. Hence, various assumptions have been proposed to ensure the identifiability, including the irreducibility assumption~\citep{scott2013classification}, the anchor point assumption~\citep{scott2015rate,liu2015classification}, the separability assumption~\citep{ramaswamy2016mixture}, etc. 

\paragraph{Best Bin Estimation.} We use Best Bin Estimation~(BBE)~\citep{garg2021mixture} as the base algorithm for class-prior estimation because it can achieve nice performance with simple implementations. First, they split the PU data into PU training data $\mathcal{D}^{\rm PTr}=\left\{\left(\bm{x}_{i}^{\rm PTr},+1\right)\right\}_{i=1}^{n^{\rm PTr}}$ and $\mathcal{D}^{\rm UTr}=\left\{\bm{x}_{i}^{\rm UTr}\right\}_{i=1}^{n^{\rm UTr}}$, and PU validation data $\mathcal{D}^{\rm PVal}=\left\{\left(\bm{x}^{\rm PVal}_{i},+1\right)\right\}_{i=1}^{n^{\rm PVal}}$ and $\mathcal{D}^{\rm UVal}=\left\{\bm{x}^{\rm UVal}_{i}\right\}_{i=1}^{n^{\rm UVal}}$. Then, they train a positive-versus-unlabeled~(PvU) classifier $f^{\rm PvU}$ with $\mathcal{D}^{\rm PTr}$ and $\mathcal{D}^{\rm UTr}$. They collect the model outputs of the PU validation data $\mathcal{Z}^{\rm P}=\left\{z^{\rm P}_{i}\right\}_{i=1}^{n^{\rm PVal}}$ and $\mathcal{Z}^{\rm U}=\left\{z^{\rm U}_{i}\right\}_{i=1}^{n^{\rm UVal}}$ where $z^{\rm P}_{i}=f^{\rm PvU}\left(\bm{x}^{\rm PVal}_{i}\right)$ and $z^{\rm U}_{i}=f^{\rm PvU}\left(\bm{x}^{\rm UVal}_{i}\right)$. They also introduce $q(z)=\int_{A_z}p(\bm{x})\diff\bm{x}$ where $A_z=\left\{\bm{x}\in\mathcal{X}|f^{\rm PvU}(\bm{x})\geq z\right\}$. Then, $q(z)$ can be regarded as the proportion of data with the model output not less than $z$. For $p\left(\bm{x}|y=+1\right)$ and $p\left(\bm{x}\right)$, they define $q^{\rm P}(z)$ and $q^{\rm U}(z)$ respectively. They empirically estimate them as
\begin{equation}
\widehat{q}^{\rm P}(z)=\frac{\sum_{i=1}^{n^{\rm PVal}}\mathbb{I}\left(f^{\rm PvU}\left(\bm{x}^{\rm PVal}_{i}\right)\geq z \right)}{n^{\rm PVal}} \textrm{~~and~~ } \widehat{q}^{\rm U}(z)=\frac{\sum_{i=1}^{n^{\rm UVal}}\mathbb{I}\left(f^{\rm PvU}\left(\bm{x}^{\rm UVal}_{i}\right)\geq z \right)}{n^{\rm UVal}}.
\end{equation}
Then, they obtain $\widehat{z}$ as
\begin{equation}
\widehat{z} = \mathop{\arg\max}_{z\in\left[0,1\right]}\left(\frac{\widehat{q}^{\rm U}(z)}{\widehat{q}^{\rm P}(z)}+\frac{1+\gamma}{\widehat{q}^{\rm P}(z)}\left(\sqrt{\frac{\ln{\left(4/\delta\right)}}{2n^{\rm PVal}}}+\sqrt{\frac{\ln{\left(4/\delta\right)}}{2n^{\rm UVal}}}\right)\right)
\end{equation}
where $\gamma$ and $\delta$ are hyperparameters respectively. Finally, they calculate the estimation value of the mixture proportion as 
\begin{equation}
\widehat{\theta}=\frac{\widehat{q}^{\rm U}\left(\widehat{z}\right)}{\widehat{q}^{\rm P}\left(\widehat{z}\right)}
\end{equation}
and they prove that $\widehat{\theta}$ is an unbiased estimator of $\theta^{*}$ when satisfying the \emph{pure positive bin assumption}, a variant of the irreducibility assumption. More detailed descriptions of the approach can be found in~\citet{garg2021mixture}. 
\paragraph{Class-prior estimation for SCARCE.}Our class-prior estimation approach is based on BBE. First, we split the complementary-label data into training and validation data. Then, we generate $q$ negative binary-class datasets $\mathcal{D}^{\rm NTr}_{k}$ and $q$ unlabeled binary-class datasets $\mathcal{D}^{\rm UTr}_{k}$ by Eq.~(\ref{neg_binary}) and Eq.~(\ref{unlabel_binary}) with training data~($k\in\mathcal{Y}$). We also generate $q$ negative binary-class datasets $\mathcal{D}^{\rm NVal}_{k}$ and $q$ unlabeled binary-class datasets $\mathcal{D}^{\rm UVal}_{k}$ by Eq.~(\ref{neg_binary}) and Eq.~(\ref{unlabel_binary}) with validation data~($k\in\mathcal{Y}$). Then, we estimate the class priors $(1-\pi_{k})$ for each label $k\in\mathcal{Y}$ using BBE adapted by swapping the positive and negative classes. Finally, we normalize $\pi_{k}$ to ensure that they sum up to one. The details of the algorithm are summarized in Algorithm~\ref{cpe_algo}. 
\begin{algorithm}
\caption{Class-prior Estimation}
\label{cpe_algo}
\noindent {\bf Input:} Complementary-label training set $\mathcal{D}$. 
\begin{algorithmic}
\FOR{$k\in\mathcal{Y}$}
\STATE {\bf Generate} training datasets $\mathcal{D}^{\rm NTr}_{k}$, $\mathcal{D}^{\rm UTr}_{k}$, validation data $\mathcal{D}^{\rm NVal}_{k}$, and $\mathcal{D}^{\rm UVal}_{k}$ by Eq.~(\ref{neg_binary}) and Eq.~(\ref{unlabel_binary});
\STATE {\bf Estimate} the value of $(1-\pi_{k})$ by employing the BBE algorithm and interchanging the positive and negative classes;
\ENDFOR
\STATE {\bf Normalize} $\pi_{k}$ to ensure they sum up to one;
\end{algorithmic}
\hspace*{0.02in} {\bf Output:} Class priors $\pi_{k}$~($k\in\mathcal{Y}$).
\end{algorithm}
\paragraph{Experimental results.}We assumed that the ground-truth class priors for all labels and datasets were 0.1, meaning that the test set was balanced. We generated complementary labels using the SCAR assumption with $\bar{\pi}_{k}=0.5$. We repeated the generation process 5 times with different random seeds. Table~\ref{cpe_res} shows the experimental results of the proposed class-prior estimation approach. We can observe that the class priors are generally accurately estimated with the proposed class-prior estimation method.
\begin{table*}[htbp]
\footnotesize
\caption{Estimated values~(mean$\pm$std) of class priors.}\label{cpe_res}
\centering
\setlength{\tabcolsep}{3pt}
\begin{tabular}{l cc cc cc cc cc} 

\toprule[1pt]	 	
 Label Index & 1 & 2 & 3& 4&5 \\ 
\midrule MNIST & 0.104$\pm$0.011 & 0.119$\pm$0.012 &  0.110$\pm$0.009 & 0.099$\pm$0.008 & 0.101$\pm$0.010   \\ 
\midrule Kuzushiji-MNIST &  0.108$\pm$0.026 & 0.098$\pm$0.011 & 0.087$\pm$0.012 & 0.104$\pm$0.004 & 0.101$\pm$0.021   \\ 
\midrule Fashion-MNIST & 0.091$\pm$0.016 & 0.118$\pm$0.005 & 0.090$\pm$0.024 & 0.090$\pm$0.009 & 0.077$\pm$0.020   \\ 
\midrule CIFAR-10 & 0.085$\pm$0.016 & 0.102$\pm$0.039 & 0.073$\pm$0.019 & 0.109$\pm$0.047 & 0.100$\pm$0.031 \\ 
\bottomrule[1pt]
\toprule[1pt]	 	
 Label Index & 6 & 7 & 8& 9&10 \\ 
\midrule MNIST & 0.087$\pm$0.007 & 0.089$\pm$0.005 & 0.106$\pm$0.019 & 0.091$\pm$0.008 & 0.096$\pm$0.016   \\  
\midrule Kuzushiji-MNIST & 0.095$\pm$0.010 & 0.105$\pm$0.025 & 0.095$\pm$0.007 & 0.094$\pm$0.016 & 0.113$\pm$0.035   \\ 
\midrule Fashion-MNIST & 0.117$\pm$0.007 & 0.070$\pm$0.010 & 0.114$\pm$0.023 & 0.117$\pm$0.016 & 0.117$\pm$0.016   \\ 
\midrule CIFAR-10 & 0.098$\pm$0.013 & 0.115$\pm$0.023 & 0.120$\pm$0.033 & 0.097$\pm$0.041 & 0.100$\pm$0.013 \\ 
\bottomrule[1pt]
\end{tabular}
\end{table*}
\section{Discussions}
\subsection{Related Work on PU Learning}
PU learning methods can be roughly categorized into two groups: sample selection methods and cost-sensitive methods. Sample selection methods try to identify reliable negative examples from the unlabeled dataset and then use supervised learning methods to learn the classifier~\citep{wang2023beyond,dai2023gradpu,garg2021mixture}. Cost-sensitive methods are based on an unbiased risk estimator, which rewrites the classification risk as that only on positive and unlabeled data~\citep{kiryo2017positive,jiang2023positive,zhao2022distpu}.
\subsection{Data Generation Process}\label{data_generation_process}
The generation process of complementary labels is summarized in Algorithm~\ref{data_gen_code}.
\begin{algorithm}[htbp]
\caption{Data Generation Process}
\label{data_gen_code}
\noindent {\bf Input:} Marginal density $p(\bm{x})$, Posterior probability distribution $p(y|\bm{x})$, number of training data $n$, probabilities $c_{k} (k\in\mathcal{Y})$. 
\begin{algorithmic}
\STATE {\bf Initialize} a complementary-label dataset $\mathcal{D}=\emptyset$;
\FOR{$i = 1,2,\ldots,n$}
\STATE {\bf Sample} an unlabeled example $\bm{x}_i$ from $p(\bm{x})$;
\STATE {\bf Sample} the label $y_i$ from $p(y|\bm{x}_i)$;
\STATE {\bf Initialize} $\bar{Y}_i=\emptyset$; 
\FOR{$k\in\mathcal{Y}$} 
\IF{$y_i \neq k$}
\STATE {\bf Assign} $\bar{Y}_i=\bar{Y}_i\cup\{k\}$ with $c_{k}$;
\ENDIF
\ENDFOR
\STATE {\bf Assign} $\mathcal{D}=\mathcal{D}\cup\{(\bm{x}_i,\bar{Y}_i)\}$;
\ENDFOR
\end{algorithmic}
\hspace*{0.02in} {\bf Output:} Complementary-label dataset $\mathcal{D}$.
\end{algorithm}
\subsection{Use of Ordinary-Label Data}
If an additional ordinary-label training set $\mathcal{D}^{\mathrm O}=\{(\bm{x}^{\mathrm O}_{i},y^{\mathrm O}_{i})\}_{i=1}^{n^{\mathrm O}}$ consisting of $n^{\mathrm O}$ training examples is available, we can include it in the training as well. Let $\widehat{R}^{\mathrm O}\left(f_{1}, f_{2}, \ldots, f_{q}\right) = \sum_{i=1}^{n^{\mathrm O}}\left(\ell\left(f_{y^{\mathrm O}_{i}}(\bm{x}^{\mathrm O}_{i})\right) + \sum_{k\in{\mathcal{Y}\backslash \{y^{\mathrm O}_{i}\}}}\ell\left(-f_{k}(\bm{x}^{\mathrm O}_{i})\right)\right)/n^{\mathrm O}$ denote the empirical risk estimator calculated on $\mathcal{D}^{\mathrm O}$, it is easy to see that $\alpha\widehat{R}\left(f_{1}, f_{2}, \ldots, f_{q}\right)+(1-\alpha)\widehat{R}^{\mathrm O}\left(f_{1}, f_{2}, \ldots, f_{q}\right)$ is still an \emph{unbiased risk estimator} w.r.t. Eq.~(\ref{ordinary_risk_ovr}) where $\alpha\in\left[0,1\right]$ is a weight.
\subsection{Summary of Assumptions}
In this section, we briefly summarize the assumptions used in this paper.
\begin{itemize}[leftmargin=1em, itemsep=1pt, topsep=1pt, parsep=-1pt]
\item The SCAR assumption, which is used to describe the data generation process.
\item The irreducibility assumption, which is used to allow the estimation of class priors.
\item The assumptions used to prove the classification-calibrated property in Theorem~\ref{calibration}, i.e. $\ell$ is convex, bounded below, differential, and satisfies $\ell(z)\leq \ell(-z)$ when $z > 0$.
\item The assumptions used to prove the estimation error bounds in Theorems~\ref{eeb}, \ref{bias_and_consistency}, and \ref{corrected_eeb}, e.g., $\sup_{f\in\mathcal{F}}\|f\|_{\infty} \leq C_{f}$, $\sup_{|z|\leq C_{f}}\ell(z) \leq C_{\ell}$, $\ell(z)$ is Lipschitz continuous, $\forall k\in\mathcal{Y}, \mathbb{E}\left[\widehat{R}^{\rm P}_{k}(f_{k})\right] \geq \beta$, the risk-correction function $g(z)$ is Lipschitz continuous, $\mathfrak{R}_{n,p}(\mathcal{F}) \leq C_{\mathfrak{R}}/\sqrt{n}$, etc.
\end{itemize}
\subsection{Limitations}
Our method is designed for instance-independent settings, where the generation of complementary labels is determined only by the ground-truth label and is not related to the feature. Moreover, if we adopt some PU learning methods that cannot share the representation layers for different labels, we may need additional storage and computational cost.
\section{Proof of Theorem~\ref{ure_cce}}\label{proof_ure_cce}
First, we introduce the following lemma.
\begin{lemma}\label{density_equality}
Under Assumption~\ref{scar}, we have $p\left(\bm{x}|\bar{y}_{k}=1\right)=p(\bm{x}|y\neq k)$.
\end{lemma}
\begin{proof}
On one hand, we have 
\begin{equation}
p(\bm{x}|\bar{y}_{k}=1,y\neq k)=\frac{p\left(\bm{x}|\bar{y}_{k}=1\right)p(y\neq k|\bm{x}, \bar{y}_{k}=1)}{p(y\neq k|\bar{y}_{k}=1)}. \nonumber
\end{equation}
According to the definition of complementary labels, we have $p(y\neq k|\bm{x}, \bar{y}_{k}=1)=p(y\neq k|\bar{y}_{k}=1)=1$. Therefore, we have $p(\bm{x}|\bar{y}_{k}=1,y\neq k)=p\left(\bm{x}|\bar{y}_{k}=1\right)$.
On the other hand, we have 
\begin{equation}
p(\bm{x}|\bar{y}_{k}=1,y\neq k)=\frac{p(\bm{x}|y\neq k)p(\bar{y}_{k}=1|\bm{x}, y\neq k)}{p(\bar{y}_{k}=1|y\neq k)}=p(\bm{x}|y\neq k), \nonumber
\end{equation}
where the first equation is derived from Assumption~\ref{scar}. The proof is completed.
\end{proof}
Then, the proof of Theorem~\ref{ure_cce} is given.

\begin{proof}[Proof of Theorem~\ref{ure_cce}]
\begin{align}
R(f)=& \mathbb{E}_{p(\bm{x},y)}\left[\mathcal{L}(f(\bm{x}), y)\right] \nonumber \\
=&\sum_{k=1}^{q}\left(\pi_{k}\mathbb{E}_{p(\bm{x}|y=k)}\left[\mathcal{L}(f(\bm{x}), k)\right]\right)\nonumber \\
=&\sum_{k=1}^{q}\left(\mathbb{E}_{p(\bm{x})}\left[\mathcal{L}(f(\bm{x}), k)\right]-(1-\pi_{k})\mathbb{E}_{p(\bm{x}|y\neq k)}\left[\mathcal{L}(f(\bm{x}), k)\right]\right)\nonumber \\
=&\sum_{k=1}^{q}\left(\mathbb{E}_{p(\bm{x})}\left[\mathcal{L}(f(\bm{x}), k)\right]-(1-\pi_{k})\mathbb{E}_{p(\bm{x}|\bar{y}_{k}=1)}\left[\mathcal{L}(f(\bm{x}), k)\right]\right)\nonumber \\
=&\sum_{k=1}^{q}\left(\mathbb{E}_{p\left(\bm{x}|\bar{y}_{k}=1\right)}\left[\left(\bar{\pi}_{k}+\pi_{k}-1\right)\mathcal{L}(f(\bm{x}), k)\right]+\mathbb{E}_{p\left(\bm{x}|\bar{y}_{k}=0\right)}\left[\left(1-\bar{\pi}_{k}\right)\mathcal{L}(f(\bm{x}), k)\right]\right), \nonumber
\end{align}
which concludes the proof.
\end{proof}
\section{Proof of Theorem~\ref{ure}}\label{proof_ure}
\begin{proof}
\begin{align}
R(f_{1}, f_{2}, \ldots, f_{q}) =& \mathbb{E}_{p(\bm{x},y)}\left[\ell\left(f_{y}\left(\bm{x}\right)\right) + \sum_{k=1, k\neq y}^{q}\ell\left(-f_{k}\left(\bm{x}\right)\right)\right] \nonumber \\
=&\mathbb{E}_{p(\bm{x},y)}\left[\sum_{k=1}^{q}\left(\mathbb{I}(k=y)\ell(f_{k}(\bm{x}))+\mathbb{I}(k\neq y)\ell(-f_{k}(\bm{x}))\right)\right] \nonumber \\
=&\sum_{k=1}^{q}\mathbb{E}_{p(\bm{x},y)}\left[\mathbb{I}(k=y)\ell\left(f_{k}(\bm{x})\right)+\mathbb{I}(k\neq y)\ell\left(-f_{k}(\bm{x})\right)\right] \nonumber \\
=&\sum_{k=1}^{q}\left(\pi_{k}\mathbb{E}_{p(\bm{x}|y=k)}\left[\ell\left(f_{k}(\bm{x})\right)\right]+\left(1-\pi_{k}\right)\mathbb{E}_{p(\bm{x}|y\neq k)}\left[\ell\left(-f_{k}(\bm{x})\right)\right]\right) \nonumber \\
=&\sum_{k=1}^{q}\left(\mathbb{E}_{p(\bm{x})}\left[\ell(f_{k}(\bm{x}))\right] - (1-\pi_{k})\mathbb{E}_{p(\bm{x}|y\neq k)}\left[\ell(f_{k}(\bm{x}))\right]\right. \nonumber \\
&\left.+(1-\pi_{k})\mathbb{E}_{p(\bm{x}|y\neq k)}\left[\ell(-f_{k}(\bm{x}))\right]\right) \nonumber \\
=&\sum_{k=1}^{q}\left(\mathbb{E}_{p(\bm{x})}\left[\ell(f_{k}(\bm{x}))\right] - (1-\pi_{k})\mathbb{E}_{p\left(\bm{x}|\bar{y}_{k}=1\right)}\left[\ell(f_{k}(\bm{x}))\right]\right. \nonumber \\
&\left.+(1-\pi_{k})\mathbb{E}_{p\left(\bm{x}|\bar{y}_{k}=1\right)}\left[\ell(-f_{k}(\bm{x}))\right]\right) \nonumber \\
=&\sum_{k=1}^{q}\left(\bar{\pi}_{k}\mathbb{E}_{p\left(\bm{x}|\bar{y}_{k}=1\right)}\left[\ell(f_{k}(\bm{x}))\right] + (1-\bar{\pi}_{k})\mathbb{E}_{p\left(\bm{x}|\bar{y}_{k}=0\right)}\left[\ell(f_{k}(\bm{x}))\right]\right. \nonumber \\
&\left.-(1-\pi_{k})\mathbb{E}_{p\left(\bm{x}|\bar{y}_{k}=1\right)}\left[\ell(f_{k}(\bm{x}))\right]+(1-\pi_{k})\mathbb{E}_{p\left(\bm{x}|\bar{y}_{k}=1\right)}\left[\ell(-f_{k}(\bm{x}))\right]\right) \nonumber \\
=&\sum_{k=1}^{q}\left(\mathbb{E}_{p\left(\bm{x}|\bar{y}_{k}=1\right)}\left[(1-\pi_{k})\ell\left(-f_{k}(\bm{x})\right)+\left(\bar{\pi}_{k}+\pi_{k}-1\right)\ell\left(f_{k}(\bm{x})\right)\right]\right. \nonumber \\
&\left.+\mathbb{E}_{p\left(\bm{x}|\bar{y}_{k}=0\right)}\left[\left(1-\bar{\pi}_{k}\right)\ell\left(f_{k}(\bm{x})\right)\right]\right). \nonumber
\end{align}
Here, $\mathbb{I}(\cdot)$ returns 1 if predicate holds. Otherwise, $\mathbb{I}(\cdot)$ returns 0. The proof is completed.
\end{proof}
\section{Proof of Lemma~\ref{nu_risk}}
\begin{proof}
Under the binary classification setting, we have the marginal density $p(\bm{x})=\pi_{+} p(\bm{x}|y=+1)+(1-\pi_{+})p(\bm{x}|y=-1)$, where $p(\bm{x}|y=+1)$ and $p(\bm{x}|y=-1)$ denote the densities of positive and negative data respectively. Then, we have the classification risk of binary classification
\begin{align}
R(f)=& \mathbb{E}_{p(\bm{x},y)}\left[\ell\left(yf^{\rm NU}(\bm{x})\right)\right] \nonumber \\
=&\pi_{+}\mathbb{E}_{p(\bm{x}|y=+1)}\left[\ell\left(f^{\rm NU}(\bm{x})\right)\right]+(1-\pi_{+})\mathbb{E}_{p(\bm{x}|y=-1)}\left[\ell\left(-f^{\rm NU}(\bm{x})\right)\right] 
\nonumber \\
=&\mathbb{E}_{p(\bm{x})}\left[\ell\left(f^{\rm NU}(\bm{x})\right)\right]-(1-\pi_{+})\mathbb{E}_{p(\bm{x}|y=-1)}\left[\ell\left(f^{\rm NU}(\bm{x})\right)\right]+(1-\pi_{+})\mathbb{E}_{p(\bm{x}|y=-1)}\left[\ell\left(-f^{\rm NU}(\bm{x})\right)\right] \nonumber \\
=&\bar{\pi}_{-}\mathbb{E}_{p(\bm{x}|\bar{y}=-1)}\left[\ell\left(f^{\rm NU}(\bm{x})\right)\right]+(1-\bar{\pi}_{-})\mathbb{E}_{p(\bm{x}|\bar{y}=0)}\left[\ell\left(f^{\rm NU}(\bm{x})\right)\right] \nonumber \\
&+(1-\pi_{+})\mathbb{E}_{p(\bm{x}|y=-1)}\left[\ell\left(-f^{\rm NU}(\bm{x})\right)-\ell\left(f^{\rm NU}(\bm{x})\right)\right]
\nonumber \\
=&\mathbb{E}_{p(\bm{x}|\bar{y}=-1)}\left[(1-\pi_{+})\ell\left(-f^{\rm NU}(\bm{x})\right)+(\bar{\pi}_{-}+\pi_{+}-1)\ell\left(f^{\rm NU}(\bm{x})\right)\right]+\mathbb{E}_{p(\bm{x}|\bar{y}=0)}\left[(1-\bar{\pi}_{-})\ell\left(f^{\rm NU}(\bm{x})\right)\right], \nonumber
\end{align}
where the last equation results from Lemma~\ref{density_equality} under the binary-class setting. The proof is completed. 
\end{proof}
\section{Proof of Theorem~\ref{calibration}}\label{proof_calibration}
To begin with, we show the following theoretical results about infinite-sample consistency from~\citet{zhang2004statistical}. For ease of notations, let $\bm{f}(\bm{x})=\left[f_{1}(\bm{x}), f_{2}(\bm{x}),\ldots,f_{q}(\bm{x})\right]$ denote the vector form of modeling outputs of all the binary classifiers. First, we elaborate the infinite-sample consistency property of the OVR strategy.
\begin{theorem}[Theorem 10 of~\citet{zhang2004statistical}]
Consider the OVR strategy, whose surrogate loss function is defined as $\Psi_y\left(\bm{f}(\bm{x})\right) = \ell\left(f_y(\bm{x})\right)+\sum_{k\in{\mathcal{Y}\backslash \{y\}}}\ell\left(-f_k(\bm{x})\right)$. Assume $\ell$ is convex, bounded below, differentiable, and $\ell(z)<\ell(-z)$ when $z>0$. Then, the OVR strategy is infinite-sample consistent on $\Omega = \mathbb{R}^K$ with respect to the 0-1 classification risk. 
\end{theorem}
Then, we elaborate the relationship between the minimum classification risk of an infinite-sample consistent method and the Bayes error.
\begin{theorem}[Theorem 3 of~\citet{zhang2004statistical}]
Let $\mathcal{B}$ be the set of all vector Borel measurable functions, which take values in $\mathbb{R}^q$. For $\Omega\subset\mathbb{R}^q$, let $\mathcal{B}_{\Omega} = \{\bm{f}\in\mathcal{B}:\forall\bm{x},\bm{f}(x)\in\Omega\}$. If $\Psi_y(\cdot)$ is infinite-sample consistent on $\Omega$ with respect to the 0-1 classification risk, then for any  $\epsilon_1>0$, there exists an $\epsilon_2>0$ such that for all underlying Borel probability measurable $p$, and $\bm{f}(\cdot)\in\mathcal{B}_{\Omega}$,
\begin{equation}
\mathbb{E}_{(\bm{x},y)\sim p}\left[\Psi_y(\bm{f}(\bm{x}))\right]\leq\inf_{\bm{f}'\in\mathcal{B}_{\Omega}}\mathbb{E}_{(\bm{x},y)\sim p}\left[\Psi_y(\bm{f}'(\bm{x}))\right] + \epsilon_2
\end{equation}
implies
\begin{equation}
    R_{\mathrm{0-1}}\left(T(\bm{f}(\cdot))\right)\leq R_{\mathrm{0-1}}^{*} + \epsilon_1,
\end{equation}
where $T(\cdot)$ is defined as $T(\bm{f}(\bm{x})):=\mathop{\arg\max}_{k=1,\dots,q}f_k(\bm{x})$.
\end{theorem}
Then, we give the proof of Theorem~\ref{calibration}.
\begin{proof}[Proof of Theorem~\ref{calibration}]
According to Theorem~\ref{ure}, the proposed classification risk $R(f_{1}, f_{2}, \ldots, f_{q})$ is equivalent to the OVR risk. Therefore, it is sufficient to elaborate the theoretical properties of the OVR risk to prove Theorem ~\ref{calibration}.
\end{proof}
\section{Proof of Theorem~\ref{eeb}}\label{proof_eeb}
First, we give the definition of Rademacher complexity.
\begin{definition}[Rademacher complexity] Let $\mathcal{X}_{n}=\{\bm{x}_{1}, \ldots \bm{x}_{n}\}$ denote $n$ i.i.d. random variables drawn from a probability distribution with density $p(\bm{x})$, $\mathcal{F}=\{f:\mathcal{X}\mapsto \mathbb{R}\}$ denote a class of measurable functions, and $\bm{\sigma}=(\sigma_{1}, \sigma_{2}, \ldots, \sigma_{n})$ denote Rademacher variables taking values from $\{+1, -1\}$ uniformly. Then, the (expected) Rademacher complexity of $\mathcal{F}$ is defined as
\begin{equation}
\mathfrak{R}_{n,p}(\mathcal{F})=\mathbb{E}_{\mathcal{X}_{n}} \mathbb{E}_{\bm{\sigma}}\left[\sup _{f \in \mathcal{F}} \frac{1}{n} \sum_{i=1}^{n} \sigma_{i} f(\bm{x}_{i})\right].
\end{equation}
\end{definition}
For ease of notation, we define $\bar{\mathcal{D}}=\mathcal{D}^{\rm U}_{1}\bigcup\mathcal{D}^{\rm U}_{2}\bigcup\ldots\bigcup\mathcal{D}^{\rm U}_{q}\bigcup\mathcal{D}^{\rm N}_{1}\bigcup\mathcal{D}^{\rm N}_{2}\bigcup\ldots\bigcup\mathcal{D}^{\rm N}_{q}$ denote the set of all the binary-class training data.
Then, we have the following lemma.
\begin{lemma}\label{eeb_lemma}
For any $\delta > 0$, the inequalities below hold with probability at least $1-\delta$:
\begin{align}
&\sup _{f_{1}, f_{2}, \ldots, f_{q} \in \mathcal{F}}\left|R(f_{1}, f_{2}, \ldots, f_{q})-\widehat{R}(f_{1}, f_{2}, \ldots, f_{q})\right| \leq \sum_{k=1}^{q}\left(
(1-\bar{\pi}_{k})C_{\ell}\sqrt{\frac{\ln{\left(2/\delta\right)}}{2n^{\rm U}_{k}}}\right. \nonumber \\
&\left.+(2-2\bar{\pi}_{k})L_{\ell}\mathfrak{R}_{n^{\rm U}_{k},p^{\rm U}_{k}}(\mathcal{F})+(4-4\pi_{k}-2\bar{\pi}_{k})L_{\ell}\mathfrak{R}_{n^{\rm N}_{k},p^{\rm N}_{k}}(\mathcal{F})+(2-2\pi_{k}-\bar{\pi}_{k})C_{\ell}\sqrt{\frac{\ln{\left(2/\delta\right)}}{2n^{\rm N}_{k}}}\right).
\end{align}
\end{lemma}
\begin{proof}
In the following proofs, we consider a general case where all the datasets $\mathcal{D}^{\rm N}_{k}$ and $\mathcal{D}^{\rm U}_{k}$ are \emph{mutually independent}. We can observe that when an unlabeled example $\bm{x}^{\rm U}_{k,i} \in \mathcal{D}^{\rm U}_{k}$ is substituted by another unlabeled example $\bm{x}^{\rm U}_{k,j}$, the value of $\sup _{f_{1}, f_{2}, \ldots, f_{q} \in \mathcal{F}}\left|R(f_{1}, f_{2}, \ldots, f_{q})-\widehat{R}(f_{1}, f_{2}, \ldots, f_{q})\right|$ changes at most $(1-\bar{\pi}_{k})C_{\ell}/n^{\rm U}_{k}$. Besides, when a negative example $\bm{x}^{\rm N}_{k,i} \in \mathcal{D}^{\rm N}_{k}$ is substituted by another negative example $\bm{x}^{\rm N}_{k,j}$, the value of $\sup _{f_{1}, f_{2}, \ldots, f_{q} \in \mathcal{F}}\left|R(f_{1}, f_{2}, \ldots, f_{q})-\widehat{R}(f_{1}, f_{2}, \ldots, f_{q})\right|$ changes at most $(2-2\pi_{k}-\bar{\pi}_{k})C_{\ell}/n^{\rm N}_{k}$. According to the McDiarmid’s inequality, for any $\delta > 0$, the following inequality holds with probability at least $1-\delta/2$:
\begin{align} \label{lemma_eq1}
&\sup _{f_{1}, f_{2}, \ldots, f_{q} \in \mathcal{F}}\left(R(f_{1}, f_{2}, \ldots, f_{q})-\widehat{R}(f_{1}, f_{2}, \ldots, f_{q})\right) \nonumber\\
\leq &\mathbb{E}_{\bar{\mathcal{D}}}\left[\sup _{f_{1}, f_{2}, \ldots, f_{q} \in \mathcal{F}}\left(R(f_{1}, f_{2}, \ldots, f_{q})-\widehat{R}(f_{1}, f_{2}, \ldots, f_{q})\right)\right] \nonumber \\
&+\sum_{k=1}^{q}\left((1-\bar{\pi}_{k})C_{\ell}\sqrt{\frac{\ln{\left(2/\delta\right)}}{2n^{\rm U}_{k}}}+(2-2\pi_{k}-\bar{\pi}_{k})C_{\ell}\sqrt{\frac{\ln{\left(2/\delta\right)}}{2n^{\rm N}_{k}}}\right),
\end{align}
where the inequality is deduced since $\sqrt{a+b}\leq\sqrt{a}+\sqrt{b}$. It is a routine work to show by symmetrization~\citep{mohri2012foundations} that
\begin{align} \label{lemma_eq2}
&\mathbb{E}_{\bar{\mathcal{D}}}\left[\sup _{f_{1}, f_{2}, \ldots, f_{q} \in \mathcal{F}}\left(R(f_{1}, f_{2}, \ldots, f_{q})-\widehat{R}(f_{1}, f_{2}, \ldots, f_{q})\right)\right] \nonumber \\
\leq &\sum_{k=1}^{q}\left(
(2-2\bar{\pi}_{k})\mathfrak{R}_{n^{\rm U}_{k},p^{\rm U}_{k}}(\ell\circ\mathcal{F})+(4-4\pi_{k}-2\bar{\pi}_{k})\mathfrak{R}_{n^{\rm N}_{k},p^{\rm N}_{k}}(\ell\circ\mathcal{F})\right),
\end{align}
where $\mathfrak{R}_{n,p}(\ell\circ\mathcal{F})$ is the Rademacher complexity of the composite function class $(\ell\circ\mathcal{F})$. According to Talagrand’s contraction lemma~\citep{ledoux1991probability}, we have
\begin{align}
\mathfrak{R}_{n^{\rm U}_{k},p^{\rm U}_{k}}(\ell\circ\mathcal{F})\leq L_{\ell}\mathfrak{R}_{n^{\rm U}_{k},p^{\rm U}_{k}}(\mathcal{F}),\label{lemma_eq3} \\
\mathfrak{R}_{n^{\rm N}_{k},p^{\rm N}_{k}}(\ell\circ\mathcal{F})\leq L_{\ell}\mathfrak{R}_{n^{\rm N}_{k},p^{\rm N}_{k}}(\mathcal{F}).\label{lemma_eq4}
\end{align}
By combining Inequality~(\ref{lemma_eq1}), Inequality~(\ref{lemma_eq2}), Inequality~(\ref{lemma_eq3}), and Inequality~(\ref{lemma_eq4}), the following inequality holds with probability at least $1-\delta/2$:
\begin{align}\label{lemma_eq5}
&\sup _{f_{1}, f_{2}, \ldots, f_{q} \in \mathcal{F}}\left(R(f_{1}, f_{2}, \ldots, f_{q})-\widehat{R}(f_{1}, f_{2}, \ldots, f_{q})\right) \leq \sum_{k=1}^{q}\left(
(2-2\bar{\pi}_{k})L_{\ell}\mathfrak{R}_{n^{\rm U}_{k},p^{\rm U}_{k}}(\mathcal{F})\right. \nonumber \\
&\left.+(1-\bar{\pi}_{k})C_{\ell}\sqrt{\frac{\ln{\left(2/\delta\right)}}{2n^{\rm U}_{k}}}+(4-4\pi_{k}-2\bar{\pi}_{k})L_{\ell}\mathfrak{R}_{n^{\rm N}_{k},p^{\rm N}_{k}}(\mathcal{F})+(2-2\pi_{k}-\bar{\pi}_{k})C_{\ell}\sqrt{\frac{\ln{\left(2/\delta\right)}}{2n^{\rm N}_{k}}}\right).
\end{align}
In the same way, we have the following inequality with probability at least $1-\delta/2$:
\begin{align}\label{lemma_eq6}
&\sup _{f_{1}, f_{2}, \ldots, f_{q} \in \mathcal{F}}\left(\widehat{R}(f_{1}, f_{2}, \ldots, f_{q})-R(f_{1}, f_{2}, \ldots, f_{q})\right) \leq \sum_{k=1}^{q}\left(
(1-\bar{\pi}_{k})C_{\ell}\sqrt{\frac{\ln{\left(2/\delta\right)}}{2n^{\rm U}_{k}}}\right. \nonumber \\
&\left.+(2-2\bar{\pi}_{k})L_{\ell}\mathfrak{R}_{n^{\rm U}_{k},p^{\rm U}_{k}}(\mathcal{F})+(4-4\pi_{k}-2\bar{\pi}_{k})L_{\ell}\mathfrak{R}_{n^{\rm N}_{k},p^{\rm N}_{k}}(\mathcal{F})+(2-2\pi_{k}-\bar{\pi}_{k})C_{\ell}\sqrt{\frac{\ln{\left(2/\delta\right)}}{2n^{\rm N}_{k}}}\right).
\end{align}
By combining Inequality~(\ref{lemma_eq5}) and Inequality~(\ref{lemma_eq6}), we have the following inequality with probability at least $1-\delta$: 
\begin{align}
&\sup _{f_{1}, f_{2}, \ldots, f_{q} \in \mathcal{F}}\left|R(f_{1}, f_{2}, \ldots, f_{q})-\widehat{R}(f_{1}, f_{2}, \ldots, f_{q})\right| \leq \sum_{k=1}^{q}\left(
(1-\bar{\pi}_{k})C_{\ell}\sqrt{\frac{\ln{\left(2/\delta\right)}}{2n^{\rm U}_{k}}}\right. \nonumber \\
&\left.+(2-2\bar{\pi}_{k})L_{\ell}\mathfrak{R}_{n^{\rm U}_{k},p^{\rm U}_{k}}(\mathcal{F})+(4-4\pi_{k}-2\bar{\pi}_{k})L_{\ell}\mathfrak{R}_{n^{\rm N}_{k},p^{\rm N}_{k}}(\mathcal{F})+(2-2\pi_{k}-\bar{\pi}_{k})C_{\ell}\sqrt{\frac{\ln{\left(2/\delta\right)}}{2n^{\rm N}_{k}}}\right),
\end{align}
which concludes the proof.
\end{proof}

\begin{proof}[Proof of Theorem~\ref{eeb}]
\begin{align}
&R(\widehat{f}_{1}, \widehat{f}_{2}, \ldots, \widehat{f}_{q}) - R(f_{1}^{*},f_{2}^{*},\ldots,f_{q}^{*}) \nonumber \\
=&R(\widehat{f}_{1}, \widehat{f}_{2}, \ldots, \widehat{f}_{q})-\widehat{R}(\widehat{f}_{1}, \widehat{f}_{2}, \ldots, \widehat{f}_{q})+\widehat{R}(\widehat{f}_{1}, \widehat{f}_{2}, \ldots, \widehat{f}_{q})-\widehat{R}(f_{1}^{*},f_{2}^{*},\ldots,f_{q}^{*}) \nonumber \\
&+\widehat{R}(f_{1}^{*},f_{2}^{*},\ldots,f_{q}^{*})- R(f_{1}^{*},f_{2}^{*},\ldots,f_{q}^{*}) \nonumber \\
\leq & R(\widehat{f}_{1}, \widehat{f}_{2}, \ldots, \widehat{f}_{q})-\widehat{R}(\widehat{f}_{1}, \widehat{f}_{2}, \ldots, \widehat{f}_{q})+\widehat{R}(f_{1}^{*},f_{2}^{*},\ldots,f_{q}^{*})- R(f_{1}^{*},f_{2}^{*},\ldots,f_{q}^{*}) \nonumber \\
\leq& 2\sup _{f_{1}, f_{2}, \ldots, f_{q} \in \mathcal{F}}\left|R(f_{1}, f_{2}, \ldots, f_{q})-\widehat{R}(f_{1}, f_{2}, \ldots, f_{q})\right| \label{thm2_proof_ineq}
\end{align}
The first inequality is deduced because $(\widehat{f}_{1}, \widehat{f}_{2}, \ldots, \widehat{f}_{q})$ is the minimizer of $\widehat{R}(f_1, f_2, \ldots, f_q)$. Combining Inequality~(\ref{thm2_proof_ineq}) and Lemma~\ref{eeb_lemma}, the proof is completed.
\end{proof}

\section{Proof of Theorem~\ref{bias_and_consistency}}\label{proof_bias_and_consistency}
Let $\mathfrak{D}^{+}_{k}(f_{k})=\left\{\left(\mathcal{D}^{\rm N}_{k}, \mathcal{D}^{\rm U}_{k}\right)|\widehat{R}^{\rm P}_{k}(f_{k}) \geq 0\right\}$ and $\mathfrak{D}^{-}_{k}(f_{k})=\left\{\left(\mathcal{D}^{\rm N}_{k}, \mathcal{D}^{\rm U}_{k}\right)|\widehat{R}^{\rm P}_{k}(f_{k}) < 0\right\}$ denote the sets of NU data pairs having positive and negative empirical risk respectively. Then we have the following lemma.
\begin{lemma}
The probability measure of $\mathfrak{D}^{-}_{k}(f_{k})$ can be bounded as follows: 
\begin{equation}
\mathbb{P}\left(\mathfrak{D}^{-}_{k}(f_{k})\right) \leq \exp{\left(\frac{-2\beta^{2}}{(1 - \pi_{k}-\bar{\pi}_{k})^{2}C_{\ell}^{2}/n^{\rm N}_{k}+(1-\bar{\pi}_{k})^{2}C_{\ell}^{2}/n^{\rm U}_{k}}\right)}.
\end{equation}
\end{lemma}
\begin{proof}
Let 
\begin{equation}
p\left(\mathcal{D}^{\rm N}_{k}\right)=p\left(\bm{x}_{k,1}^{\mathrm{N}}|\bar{y}_{k}=1\right)p\left(\bm{x}_{k,2}^{\mathrm{N}}|\bar{y}_{k}=1\right)\ldots p\left(\bm{x}_{k,n^{\rm N}_{k}}^{\mathrm{N}}|\bar{y}_{k}=1\right) \nonumber
\end{equation}
and
\begin{equation}
p\left(\mathcal{D}^{\rm U}_{k}\right)=p\left(\bm{x}_{k,1}^{\mathrm{U}}|\bar{y}_{k}=0\right)p\left(\bm{x}_{k,2}^{\mathrm{U}}|\bar{y}_{k}=0\right)\ldots p\left(\bm{x}_{k,n^{\rm U}_{k}}^{\mathrm{U}}|\bar{y}_{k}=0\right) \nonumber
\end{equation}
denote the probability density of $\mathcal{D}^{\rm N}_{k}$ and $\mathcal{D}^{\rm U}_{k}$ respectively. The joint probability density of $\mathcal{D}^{\rm N}_{k}$ and $\mathcal{D}^{\rm U}_{k}$ is 
\begin{equation}
p\left(\mathcal{D}^{\rm N}_{k}, \mathcal{D}^{\rm U}_{k}\right)=p\left(\mathcal{D}^{\rm N}_{k}\right)p\left(\mathcal{D}^{\rm U}_{k}\right). \nonumber
\end{equation}
Then, the probability measure $\mathbb{P}\left(\mathfrak{D}^{-}_{k}(f_{k})\right)$ is defined as 
\begin{align}
\mathbb{P}\left(\mathfrak{D}^{-}_{k}(f_{k})\right)=&\int_{\left(\mathcal{D}^{\rm N}_{k}, \mathcal{D}^{\rm U}_{k}\right)\in \mathfrak{D}^{-}_{k}\left(f_{k}\right)}p\left(\mathcal{D}^{\rm N}_{k}, \mathcal{D}^{\rm U}_{k}\right)\diff \left(\mathcal{D}^{\rm N}_{k}, \mathcal{D}^{\rm U}_{k}\right) \nonumber \\
=&\int_{\left(\mathcal{D}^{\rm N}_{k}, \mathcal{D}^{\rm U}_{k}\right)\in \mathfrak{D}^{-}_{k}(f_{k})}p\left(\mathcal{D}^{\rm N}_{k}, \mathcal{D}^{\rm U}_{k}\right)\diff \bm{x}_{k,1}^{\mathrm{N}}\ldots \diff\bm{x}_{k,n^{\rm N}_{k}}^{\mathrm{N}}\diff\bm{x}_{k,1}^{\mathrm{U}}\ldots \diff\bm{x}_{k,n^{\rm U}_{k}}^{\mathrm{U}}. \nonumber
\end{align}
When a negative example in $\mathcal{D}^{\rm N}_{k}$ is substituted by another different negative example, the change of the value of $\widehat{R}^{\rm P}_{k}(f_{k})$ is no more than $(1 - \pi_{k}-\bar{\pi}_{k})C_{\ell}/n^{\rm N}_{k}$; when an unlabeled example in $\mathcal{D}^{\rm U}_{k}$ is substituted by another different unlabeled example, the change of the value of $\widehat{R}^{\rm P}_{k}(f_{k})$ is no more than $(1-\bar{\pi}_{k})C_{\ell}/n^{\rm U}_{k}$. Therefore, by applying the McDiarmid’s inequality, we can obtain the following inequality:
\begin{equation}
\mathbb{P}\left(\mathbb{E}\left[\widehat{R}^{\rm P}_{k}(f_{k})\right] - \widehat{R}^{\rm P}_{k}(f_{k}) \geq \beta\right) \leq \exp{\left(\frac{-2\beta^{2}}{(1 - \pi_{k}-\bar{\pi}_{k})^{2}C_{\ell}^{2}/n^{\rm N}_{k}+(1-\bar{\pi}_{k})^{2}C_{\ell}^{2}/n^{\rm U}_{k}}\right)}.
\end{equation}
Therefore, we have
\begin{align}
\mathbb{P}\left(\mathfrak{D}^{-}_{k}(f_{k})\right) = & \mathbb{P}\left(\widehat{R}^{\rm P}_{k}(f_{k}) \leq 0\right) \nonumber \\
\leq & \mathbb{P}\left(\widehat{R}^{\rm P}_{k}(f_{k}) \leq \mathbb{E}\left[\widehat{R}^{\rm P}_{k}(f_{k})\right] - \beta\right) \nonumber \\
= & \mathbb{P}\left(\mathbb{E}\left[\widehat{R}^{\rm P}_{k}(f_{k})\right] - \widehat{R}^{\rm P}_{k}(f_{k}) \geq \beta\right) \nonumber \\
\leq & \exp{\left(\frac{-2\beta^{2}}{(1 - \pi_{k}-\bar{\pi}_{k})^{2}C_{\ell}^{2}/n^{\rm N}_{k}+(1-\bar{\pi}_{k})^{2}C_{\ell}^{2}/n^{\rm U}_{k}}\right)},
\end{align}
which concludes the proof.
\end{proof}
We present a more complete version of Theorem~\ref{bias_and_consistency} and its proof.
\begin{theorem}
Based on the above assumptions, the bias of the expectation of the corrected risk estimator has the following lower and upper bounds: 
\begin{equation}
0\leq\mathbb{E}[\widetilde{R}(f_{1}, f_{2}, \ldots, f_{q})]-R(f_{1}, f_{2}, \ldots, f_{q}) \leq \sum_{k=1}^{q}\left(2-2\bar{\pi}_{k}-\pi_{k}\right)\left(L_{g}+1\right)C_{\ell}\Delta_{k},
\end{equation}
where $\Delta_{k}=\exp{\left(-2\beta^{2}/\left((1 - \pi_{k}-\bar{\pi}_{k})^{2}C_{\ell}^{2}/n^{\rm N}_{k}+(1-\bar{\pi}_{k})^{2}C_{\ell}^{2}/n^{\rm U}_{k}\right)\right)}$. Furthermore, for any $\delta > 0$, the following inequality holds with probability at least $1-\delta$:
\begin{align}
&|\widetilde{R}(f_{1}, f_{2}, \ldots, f_{q})-R(f_{1}, f_{2}, \ldots, f_{q})| 
\leq \sum_{k=1}^{q}\left(\left(1-\bar{\pi}_{k}\right)C_{\ell}L_{g}\sqrt{\frac{\ln{\left(2/\delta\right)}}{2n^{\rm U}_{k}}}\right. \nonumber \\
&\left.+\left(2-2\bar{\pi}_{k}-\pi_{k}\right)\left(L_{g}+1\right)C_{\ell}\Delta_{k}+\left(\left(1-\pi_{k}-\bar{\pi}_{k}\right)L_{g}+1-\pi_{k}\right)
C_{\ell}\sqrt{\frac{\ln{\left(2/\delta\right)}}{2n^{\rm N}_{k}}}\right). \nonumber 
\end{align}
\end{theorem}
\begin{proof}
First, we have
\begin{equation}
\mathbb{E}\left[\widetilde{R}(f_{1}, f_{2}, \ldots, f_{q})\right]-R(f_{1}, f_{2}, \ldots, f_{q})= \mathbb{E}\left[\widetilde{R}(f_{1}, f_{2}, \ldots, f_{q}) -\widehat{R}(f_{1}, f_{2}, \ldots, f_{q})\right]. \nonumber \\
\end{equation}
Since $\widetilde{R}(f_{1}, f_{2}, \ldots, f_{q})$ is an upper bound of $\widehat{R}(f_{1}, f_{2}, \ldots, f_{q})$, we have 
\begin{equation}
\mathbb{E}\left[\widetilde{R}(f_{1}, f_{2}, \ldots, f_{q})\right]-R(f_{1}, f_{2}, \ldots, f_{q})\geq 0. \nonumber \\
\end{equation}
Besides, we have 
\begin{align}
& \mathbb{E}\left[\widetilde{R}(f_{1}, f_{2}, \ldots, f_{q})\right]-R(f_{1}, f_{2}, \ldots, f_{q}) \nonumber \\
= &\sum_{k=1}^{q}\int_{\left(\mathcal{D}^{\rm N}_{k}, \mathcal{D}^{\rm U}_{k}\right)\in \mathfrak{D}^{-}_{k}(f_{k})}\left(g\left(\widehat{R}^{\rm P}_{k}(f_{k})\right)-\widehat{R}^{\rm P}_{k}(f_{k})\right)p\left(\mathcal{D}^{\rm N}_{k}, \mathcal{D}^{\rm U}_{k}\right)\diff \left(\mathcal{D}^{\rm N}_{k}, \mathcal{D}^{\rm U}_{k}\right) \nonumber \\
\leq & \sum_{k=1}^{q}\sup_{\left(\mathcal{D}^{\rm N}_{k}, \mathcal{D}^{\rm U}_{k}\right)\in \mathfrak{D}^{-}_{k}(f_{k})}\left(g\left(\widehat{R}^{\rm P}_{k}(f_{k})\right)-\widehat{R}^{\rm P}_{k}(f_{k})\right)\int_{(\mathcal{D}^{\rm N}_{k}, \mathcal{D}^{\rm U})\in \mathfrak{D}^{-}_{k}(f_{k})}p\left(\mathcal{D}^{\rm N}_{k}, \mathcal{D}^{\rm U}_{k}\right)\diff \left(\mathcal{D}^{\rm N}_{k}, \mathcal{D}^{\rm U}_{k}\right) \nonumber \\
=&\sum_{k=1}^{q}\sup_{\left(\mathcal{D}^{\rm N}_{k}, \mathcal{D}^{\rm U}_{k}\right)\in \mathfrak{D}^{-}_{k}(f_{k})}\left(g\left(\widehat{R}^{\rm P}_{k}(f_{k})\right)-\widehat{R}^{\rm P}_{k}(f_{k})\right)\mathbb{P}\left(\mathfrak{D}^{-}_{k}(f_{k})\right) \nonumber \\
\leq & \sum_{k=1}^{q}\sup_{\left(\mathcal{D}^{\rm N}_{k}, \mathcal{D}^{\rm U}_{k}\right)\in \mathfrak{D}^{-}_{k}(f_{k})}(L_{g}\left|\widehat{R}^{\rm P}_{k}(f_{k})\right| + \left|\widehat{R}^{\rm P}_{k}(f_{k})\right|)\mathbb{P}\left(\mathfrak{D}^{-}_{k}(f_{k})\right). \nonumber
\end{align}
Besides, 
\begin{align}
\left|\widehat{R}^{\rm P}_{k}(f_{k})\right| =& \left|\frac{\bar{\pi}_{k}+\pi_{k}-1}{n^{\rm N}_{k}}\sum_{i=1}^{n^{\rm N}_{k}}\ell\left(f_{k}(\bm{x}_{k,i}^{\mathrm{N}})\right)+\frac{1-\bar{\pi}_{k}}{n^{\rm U}_{k}}\sum_{i=1}^{n^{\rm U}_{k}}\ell\left(f_{k}(\bm{x}_{k,i}^{\mathrm{U}})\right)\right| \nonumber \\
\leq & \left|\frac{\bar{\pi}_{k}+\pi_{k}-1}{n^{\rm N}_{k}}\sum_{i=1}^{n^{\rm N}_{k}}\ell\left(f_{k}(\bm{x}_{k,i}^{\mathrm{N}})\right)\right|+\left|\frac{1-\bar{\pi}_{k}}{n^{\rm U}_{k}}\sum_{i=1}^{n^{\rm U}_{k}}\ell\left(f_{k}(\bm{x}_{k,i}^{\mathrm{U}})\right)\right| \nonumber \\
\leq & (1 - \pi_{k}-\bar{\pi}_{k})C_{\ell} + (1-\bar{\pi}_{k})C_{\ell}=\left(2-2\bar{\pi}_{k}-\pi_{k}\right)C_{\ell}.\nonumber
\end{align}
Therefore, we have 
\begin{align}
& \mathbb{E}\left[\widetilde{R}(f_{1}, f_{2}, \ldots, f_{q})\right]-R(f_{1}, f_{2}, \ldots, f_{q}) \nonumber \\
\leq & \sum_{k=1}^{q}\sup_{\left(\mathcal{D}^{\rm N}_{k}, \mathcal{D}^{\rm U}_{k}\right)\in \mathfrak{D}^{-}_{k}(f_{k})}(L_{g}\left|\widehat{R}^{\rm P}_{k}(f_{k})\right| + \left|\widehat{R}^{\rm P}_{k}(f_{k})\right|)\mathbb{P}(\mathfrak{D}^{-}_{k}(f_{k})). \nonumber \\
\leq & \sum_{k=1}^{q}\left(2-2\bar{\pi}_{k}-\pi_{k}\right)\left(L_{g}+1\right)C_{\ell}\exp{\left(\frac{-2\beta^{2}}{(1 - \pi_{k}-\bar{\pi}_{k})^{2}C_{\ell}^{2}/n^{\rm N}_{k}+(1-\bar{\pi}_{k})^{2}C_{\ell}^{2}/n^{\rm U}_{k}}\right)} \nonumber\\
=&\sum_{k=1}^{q}\left(2-2\bar{\pi}_{k}-\pi_{k}\right)\left(L_{g}+1\right)C_{\ell}\Delta_{k}, \nonumber
\end{align}
which concludes the first part of the proof of Theorem 3. Before giving the proof for the second part, we give the upper bound of $\left|\mathbb{E}\left[\widetilde{R}(f_{1}, f_{2}, \ldots, f_{q})\right]-\widetilde{R}(f_{1}, f_{2}, \ldots, f_{q})\right|$. 
When an unlabeled example from $\mathcal{D}^{\rm U}_{k}$ is substituted by another unlabeled example, the value of $\widetilde{R}(f_{1}, f_{2}, \ldots, f_{q})$ changes at most $\left(1-\bar{\pi}_{k}\right)C_{\ell}L_{g}/n^{\rm U}_{k}$. When a negative example from $\mathcal{D}^{\rm N}_{k}$ is substituted by a different example, the value of $\widetilde{R}(f_{1}, f_{2}, \ldots, f_{q})$  changes at most $\left(\left(1-\pi_{k}-\bar{\pi}_{k}\right)L_{g}+1-\pi_{k}\right)
C_{\ell}/n^{\rm N}_{k}$. By applying McDiarmid’s inequality, we have the following inequalities with probability at least $1-\delta/2$:
\begin{align}
\widetilde{R}(f_{1}, f_{2}, \ldots, f_{q}) - \mathbb{E}\left[\widetilde{R}(f_{1}, f_{2}, \ldots, f_{q})\right] \leq& \sum_{k=1}^{q} \left(1-\bar{\pi}_{k}\right)C_{\ell}L_{g}\sqrt{\frac{\ln{\left(2/\delta\right)}}{2n^{\rm U}_{k}}} \nonumber \\
&+\sum_{k=1}^{q}\left(\left(1-\pi_{k}-\bar{\pi}_{k}\right)L_{g}+1-\pi_{k}\right)
C_{\ell}\sqrt{\frac{\ln{\left(2/\delta\right)}}{2n^{\rm N}_{k}}}
, \nonumber \\ 
\mathbb{E}\left[\widetilde{R}(f_{1}, f_{2}, \ldots, f_{q})\right] - \widetilde{R}(f_{1}, f_{2}, \ldots, f_{q})  \leq& \sum_{k=1}^{q} \left(1-\bar{\pi}_{k}\right)C_{\ell}L_{g}\sqrt{\frac{\ln{\left(2/\delta\right)}}{2n^{\rm U}_{k}}} \nonumber \\
&+\sum_{k=1}^{q}\left(\left(1-\pi_{k}-\bar{\pi}_{k}\right)L_{g}+1-\pi_{k}\right)
C_{\ell}\sqrt{\frac{\ln{\left(2/\delta\right)}}{2n^{\rm N}_{k}}}
. \nonumber 
\end{align}
Then, with probability at least $1-\delta$, we have 
\begin{align}
\left|\mathbb{E}\left[\widetilde{R}(f_{1}, f_{2}, \ldots, f_{q})\right] - \widetilde{R}(f_{1}, f_{2}, \ldots, f_{q})\right|  \leq& \sum_{k=1}^{q} \left(1-\bar{\pi}_{k}\right)C_{\ell}L_{g}\sqrt{\frac{\ln{\left(2/\delta\right)}}{2n^{\rm U}_{k}}} \nonumber \\
&+\sum_{k=1}^{q}\left(\left(1-\pi_{k}-\bar{\pi}_{k}\right)L_{g}+1-\pi_{k}\right)
C_{\ell}\sqrt{\frac{\ln{\left(2/\delta\right)}}{2n^{\rm N}_{k}}}
. \nonumber 
\end{align}
Therefore, with probability at least $1-\delta$ we have
\begin{align}
&\left|\widetilde{R}(f_{1}, f_{2}, \ldots, f_{q})-R(f_{1}, f_{2}, \ldots, f_{q})\right| \nonumber \\
=& \left|\widetilde{R}(f_{1}, f_{2}, \ldots, f_{q})-\mathbb{E}[\widetilde{R}(f_{1}, f_{2}, \ldots, f_{q})]+\mathbb{E}[\widetilde{R}(f_{1}, f_{2}, \ldots, f_{q})]-R(f_{1}, f_{2}, \ldots, f_{q})\right| \nonumber \\
\leq &\left|\widetilde{R}(f_{1}, f_{2}, \ldots, f_{q})-\mathbb{E}[\widetilde{R}(f_{1}, f_{2}, \ldots, f_{q})]\right|+\left|\mathbb{E}[\widetilde{R}(f_{1}, f_{2}, \ldots, f_{q})]-R(f_{1}, f_{2}, \ldots, f_{q})\right| \nonumber \\
=&\left|\widetilde{R}(f_{1}, f_{2}, \ldots, f_{q})-\mathbb{E}[\widetilde{R}(f_{1}, f_{2}, \ldots, f_{q})]\right|+\mathbb{E}[\widetilde{R}(f_{1}, f_{2}, \ldots, f_{q})]-R(f_{1}, f_{2}, \ldots, f_{q}) \nonumber \\
\leq & \sum_{k=1}^{q} \left(1-\bar{\pi}_{k}\right)C_{\ell}L_{g}\sqrt{\frac{\ln{\left(2/\delta\right)}}{2n^{\rm U}_{k}}}+\sum_{k=1}^{q}\left(\left(1-\pi_{k}-\bar{\pi}_{k}\right)L_{g}+1-\pi_{k}\right)
C_{\ell}\sqrt{\frac{\ln{\left(2/\delta\right)}}{2n^{\rm N}_{k}}} \nonumber \\
&+\sum_{k=1}^{q}\left(2-2\bar{\pi}_{k}-\pi_{k}\right)\left(L_{g}+1\right)C_{\ell}\Delta_{k}, \nonumber
\end{align}
which concludes the proof.
\end{proof}
\section{Proof of Theorem~\ref{corrected_eeb}}\label{proof_corrected_eeb}
In this section, we adopt an alternative definition of Rademacher complexity:
\begin{equation}
\mathfrak{R}'_{n,p}(\mathcal{F})=\mathbb{E}_{\mathcal{X}_{n}} \mathbb{E}_{\bm{\sigma}}\left[\sup _{f \in \mathcal{F}} \left|\frac{1}{n} \sum_{i=1}^{n} \sigma_{i} f(\bm{x}_{i})\right|\right].
\end{equation}
Then, we introduce the following lemmas.
\begin{lemma}
Without any composition, for any $\mathcal{F}$, we have $\mathfrak{R}'_{n,p}(\mathcal{F})\geq\mathfrak{R}_{n,p}(\mathcal{F})$. If $\mathcal{F}$ is closed under negation, we have $\mathfrak{R}'_{n,p}(\mathcal{F})=\mathfrak{R}_{n,p}(\mathcal{F})$.
\end{lemma}
\begin{lemma}[Theorem 4.12 in~\citep{ledoux1991probability}] \label{new_rademacher}
If $\psi:\mathbb{R}\rightarrow\mathbb{R}$ is a Lipschitz continuous function with a Lipschitz constant $L_{\psi}$ and satisfies $\psi(0)=0$, we have 
\begin{equation}
\mathfrak{R}'_{n,p}(\psi\circ\mathcal{F})\leq 2L_{\psi}\mathfrak{R}'_{n,p}(\mathcal{F}), \nonumber
\end{equation}
where $\psi\circ\mathcal{F}=\{\psi\circ f|f\in \mathcal{F}\}$.
\end{lemma}
Before giving the proof of Theorem~\ref{corrected_eeb}, we give the following lemma.
\begin{lemma}\label{lemma_of_thm4}
For any $\delta > 0$, the inequalities below hold with probability at least $1-\delta$:
\begin{align}
&\sup _{f_{1}, f_{2}, \ldots, f_{q} \in \mathcal{F}}\left|R(f_{1}, f_{2}, \ldots, f_{q})-\widetilde{R}(f_{1}, f_{2}, \ldots, f_{q})\right| \nonumber \\
\leq& \sum_{k=1}^{q} \left(1-\bar{\pi}_{k}\right)C_{\ell}L_{g}\sqrt{\frac{\ln{\left(1/\delta\right)}}{2n^{\rm U}_{k}}}+\sum_{k=1}^{q}\left(\left(1-\pi_{k}-\bar{\pi}_{k}\right)L_{g}+1-\pi_{k}\right)
C_{\ell}\sqrt{\frac{\ln{\left(1/\delta\right)}}{2n^{\rm N}_{k}}} \nonumber\\
&+ \sum_{k=1}^{q}\left(\left(4-4\bar{\pi}_{k}\right)L_{g}L_{\ell}\mathfrak{R}_{n^{\rm U}_{k},p^{\rm U}_{k}}(\mathcal{F})+\left(\left(4-4\pi_{k}-4\bar{\pi}_{k}\right)L_{g}+4-4\pi_{k}\right)L_{\ell}\mathfrak{R}_{n^{\rm N}_{k},p^{\rm N}_{k}}(\mathcal{F})\right)\nonumber \\
&+\sum_{k=1}^{q}\left(2-2\bar{\pi}_{k}-\pi_{k}\right)\left(L_{g}+1\right)C_{\ell}\Delta_{k}. \nonumber
\end{align}
\end{lemma}
\begin{proof}
Similar to previous proofs, we can observe that when an unlabeled example from $\mathcal{D}^{\rm U}_{k}$ is substituted by another unlabeled example, the value of $\sup _{f_{1}, f_{2}, \ldots, f_{q} \in \mathcal{F}}\left|\mathbb{E}\left[\widetilde{R}(f_{1}, f_{2}, \ldots, f_{q})\right]-\widetilde{R}(f_{1}, f_{2}, \ldots, f_{q})\right|$ changes at most $\left(1-\bar{\pi}_{k}\right)C_{\ell}L_{g}/n^{\rm U}_{k}$. When a negative example from $\mathcal{D}^{\rm N}_{k}$ is substituted by a different example, the value of $\sup _{f_{1}, f_{2}, \ldots, f_{q} \in \mathcal{F}}\left|\mathbb{E}\left[\widetilde{R}(f_{1}, f_{2}, \ldots, f_{q})\right]-\widetilde{R}(f_{1}, f_{2}, \ldots, f_{q})\right|$  changes at most $\left(\left(1-\pi_{k}-\bar{\pi}_{k}\right)L_{g}+1-\pi_{k}\right)
C_{\ell}/n^{\rm N}_{k}$. By applying McDiarmid’s inequality, we have the following inequality with probability at least $1-\delta$: 
\begin{align}
& \sup _{f_{1}, f_{2}, \ldots, f_{q} \in \mathcal{F}}\left|\mathbb{E}\left[\widetilde{R}(f_{1}, f_{2}, \ldots, f_{q})\right]-\widetilde{R}(f_{1}, f_{2}, \ldots, f_{q})\right| \nonumber \\
-& \mathbb{E}\left[\sup _{f_{1}, f_{2}, \ldots, f_{q} \in \mathcal{F}}\left|\mathbb{E}\left[\widetilde{R}(f_{1}, f_{2}, \ldots, f_{q})\right]-\widetilde{R}(f_{1}, f_{2}, \ldots, f_{q})\right|\right] \nonumber \\
\leq& \sum_{k=1}^{q} \left(1-\bar{\pi}_{k}\right)C_{\ell}L_{g}\sqrt{\frac{\ln{\left(1/\delta\right)}}{2n^{\rm U}_{k}}}+\sum_{k=1}^{q}\left(\left(1-\pi_{k}-\bar{\pi}_{k}\right)L_{g}+1-\pi_{k}\right)
C_{\ell}\sqrt{\frac{\ln{\left(1/\delta\right)}}{2n^{\rm N}_{k}}}. \label{prof_lemma_ineq_1}
\end{align}
Besides, we have 
\begin{align}
&\mathbb{E}\left[\sup _{f_{1}, f_{2}, \ldots, f_{q} \in \mathcal{F}}\left|\mathbb{E}\left[\widetilde{R}(f_{1}, f_{2}, \ldots, f_{q})\right]-\widetilde{R}(f_{1}, f_{2}, \ldots, f_{q})\right|\right]  
\nonumber \\
=&\mathbb{E}_{\bar{\mathcal{D}}}\left[\sup _{f_{1}, f_{2}, \ldots, f_{q} \in \mathcal{F}}\left|\mathbb{E}_{\bar{\mathcal{D}}'}\left[\widetilde{R}(f_{1}, f_{2}, \ldots, f_{q})\right]-\widetilde{R}(f_{1}, f_{2}, \ldots, f_{q})\right|\right] \nonumber \\
\leq & \mathbb{E}_{\bar{\mathcal{D}},\bar{\mathcal{D}}'}\left[\sup _{f_{1}, f_{2}, \ldots, f_{q} \in \mathcal{F}}\left|\widetilde{R}(f_{1}, f_{2}, \ldots, f_{q};\bar{\mathcal{D}})-\widetilde{R}(f_{1}, f_{2}, \ldots, f_{q};\bar{\mathcal{D}}')\right|\right], \label{prof_lemma_ineq_2}
\end{align}
where the last inequality is deduced by applying Jensen’s inequality twice since the absolute value function and the supremum function are both convex.
Here, $\widetilde{R}(f_{1}, f_{2}, \ldots, f_{q};\bar{\mathcal{D}})$ denotes the value of $\widetilde{R}(f_{1}, f_{2}, \ldots, f_{q})$ calculated on $\bar{\mathcal{D}}$. To ensure that the conditions in Lemma~\ref{new_rademacher} hold, we introduce $\Tilde{\ell}(z)=\ell(z)-\ell(0)$. It is obvious that $\Tilde{\ell}(0)=0$ and $\Tilde{\ell}(z)$ is also a Lipschitz continuous function with a Lipschitz constant $L_{\ell}$. Then, we have
\begin{align}
&\left|\widetilde{R}(f_{1}, f_{2}, \ldots, f_{q};\bar{\mathcal{D}})-\widetilde{R}(f_{1}, f_{2}, \ldots, f_{q};\bar{\mathcal{D}}')\right| \nonumber \\
\leq & \sum_{k=1}^{q}\left|g\left(\frac{\bar{\pi}_{k}+\pi_{k}-1}{n^{\rm N}_{k}}\sum_{i=1}^{n^{\rm N}_{k}}\ell\left(f_{k}(\bm{x}_{k,i}^{\mathrm{N}})\right)+\frac{1-\bar{\pi}_{k}}{n^{\rm U}_{k}}\sum_{i=1}^{n^{\rm U}_{k}}\ell\left(f_{k}(\bm{x}_{k,i}^{\mathrm{U}})\right)\right)\right. \nonumber \\
&\left.-g\left(\frac{\bar{\pi}_{k}+\pi_{k}-1}{n^{\rm N}_{k}}\sum_{i=1}^{n^{\rm{N}}_{k}}\ell\left(f_{k}(\bm{x}_{k,i}^{\mathrm{N}'})\right)+\frac{1-\bar{\pi}_{k}}{n^{\rm U}_{k}}\sum_{i=1}^{n^{\rm U}_{k}}\ell\left(f_{k}(\bm{x}_{k,i}^{\mathrm{U}'})\right)\right)\right| \nonumber \\
&+\sum_{k=1}^{q}\left|\frac{1-\pi_{k}}{n^{\rm N}_{k}}\sum_{i=1}^{n^{\rm N}_{k}}\ell\left(-f_{k}(\bm{x}_{k,i}^{\mathrm{N}})\right)-\frac{1-\pi_{k}}{n^{\rm N}_{k}}\sum_{i=1}^{n^{\rm N}_{k}}\ell\left(-f_{k}(\bm{x}_{k,i}^{\mathrm{N}'})\right)\right|\nonumber \\
\leq& \sum_{k=1}^{q}L_{g}\left|\frac{\bar{\pi}_{k}+\pi_{k}-1}{n^{\rm N}_{k}}\sum_{i=1}^{n^{\rm N}_{k}}\left(\ell\left(f_{k}(\bm{x}_{k,i}^{\mathrm{N}})\right)-\ell\left(f_{k}(\bm{x}_{k,i}^{\mathrm{N}'})\right)\right)+\frac{1-\bar{\pi}_{k}}{n^{\rm N}_{k}}\sum_{i=1}^{n^{\rm U}_{k}}\left(\ell\left(f_{k}(\bm{x}_{k,i}^{\mathrm{U}})\right)-\ell\left(f_{k}(\bm{x}_{k,i}^{\mathrm{U}'})\right)\right)\right| \nonumber \\
&+\sum_{k=1}^{q}\left|\frac{1-\pi_{k}}{n^{\rm N}_{k}}\sum_{i=1}^{n^{\rm N}_{k}}\left(\ell\left(-f_{k}(\bm{x}_{k,i}^{\mathrm{N}})\right)-\ell\left(-f_{k}(\bm{x}_{k,i}^{\mathrm{N}'})\right)\right)\right|. \label{proof_thm4_mid_eq}
\end{align}
Besides, we can observe $\ell(z_1)-\ell(z_2)=\Tilde{\ell}(z_1)-\Tilde{\ell}(z_2)$. Therefore, the RHS of Inequality~(\ref{proof_thm4_mid_eq}) can be expressed as
\begin{align}
& \sum_{k=1}^{q}L_{g}\left|\frac{\bar{\pi}_{k}+\pi_{k}-1}{n^{\rm N}_{k}}\sum_{i=1}^{n^{\rm N}_{k}}\left(\Tilde{\ell}\left(f_{k}(\bm{x}_{k,i}^{\mathrm{N}})\right)-\Tilde{\ell}\left(f_{k}(\bm{x}_{k,i}^{\mathrm{N}'})\right)\right)+\frac{1-\bar{\pi}_{k}}{n^{\rm N}_{k}}\sum_{i=1}^{n^{\rm U}_{k}}\left(\Tilde{\ell}\left(f_{k}(\bm{x}_{k,i}^{\mathrm{U}})\right)-\Tilde{\ell}\left(f_{k}(\bm{x}_{k,i}^{\mathrm{U}'})\right)\right)\right| \nonumber \\
&+\sum_{k=1}^{q}\left|\frac{1-\pi_{k}}{n^{\rm N}_{k}}\sum_{i=1}^{n^{\rm N}_{k}}\left(\Tilde{\ell}\left(-f_{k}(\bm{x}_{k,i}^{\mathrm{N}})\right)-\Tilde{\ell}\left(-f_{k}(\bm{x}_{k,i}^{\mathrm{N}'})\right)\right)\right|.\nonumber 
\end{align}
Then, it is a routine work to show by symmetrization~\citep{mohri2012foundations} that
\begin{align}
&\mathbb{E}_{\bar{\mathcal{D}},\bar{\mathcal{D}}'}\left[\sup _{f_{1}, f_{2}, \ldots, f_{q} \in \mathcal{F}}\left|\widetilde{R}(f_{1}, f_{2}, \ldots, f_{q};\bar{\mathcal{D}})-\widetilde{R}(f_{1}, f_{2}, \ldots, f_{q};\bar{\mathcal{D}}')\right|\right] \nonumber \\
\leq& \sum_{k=1}^{q}\left(\left(2-2\bar{\pi}_{k}\right)L_{g}\mathfrak{R}'_{n^{\rm U}_{k},p^{\rm U}_{k}}(\Tilde{\ell}\circ\mathcal{F})+\left(\left(2-2\pi_{k}-2\bar{\pi}_{k}\right)L_{g}+2-2\pi_{k}\right)\mathfrak{R}'_{n^{\rm N}_{k},p^{\rm N}_{k}}(\Tilde{\ell}\circ\mathcal{F})\right) \nonumber \\
\leq&\sum_{k=1}^{q}\left(\left(4-4\bar{\pi}_{k}\right)L_{g}L_{\ell}\mathfrak{R}'_{n^{\rm U}_{k},p^{\rm U}_{k}}(\mathcal{F})+\left(\left(4-4\pi_{k}-4\bar{\pi}_{k}\right)L_{g}+4-4\pi_{k}\right)L_{\ell}\mathfrak{R}'_{n^{\rm N}_{k},p^{\rm N}_{k}}(\mathcal{F})\right) \nonumber \\
=&\sum_{k=1}^{q}\left(\left(4-4\bar{\pi}_{k}\right)L_{g}L_{\ell}\mathfrak{R}_{n^{\rm U}_{k},p^{\rm U}_{k}}(\mathcal{F})+\left(\left(4-4\pi_{k}-4\bar{\pi}_{k}\right)L_{g}+4-4\pi_{k}\right)L_{\ell}\mathfrak{R}_{n^{\rm N}_{k},p^{\rm N}_{k}}(\mathcal{F})\right), \label{prof_lemma_ineq_3}
\end{align}
where the second inequality is deduced according to Lemma~\ref{new_rademacher} and the last equality is based on the assumption that $\mathcal{F}$ is closed under negation.
By combing Inequality~(\ref{prof_lemma_ineq_1}), Inequality~(\ref{prof_lemma_ineq_2}), and Inequality~(\ref{prof_lemma_ineq_3}), we have the following inequality with probability at least $1-\delta$:
\begin{align}
& \sup _{f_{1}, f_{2}, \ldots, f_{q} \in \mathcal{F}}\left|\mathbb{E}\left[\widetilde{R}(f_{1}, f_{2}, \ldots, f_{q})\right]-\widetilde{R}(f_{1}, f_{2}, \ldots, f_{q})\right| \nonumber \\
\leq& \sum_{k=1}^{q} \left(1-\bar{\pi}_{k}\right)C_{\ell}L_{g}\sqrt{\frac{\ln{\left(1/\delta\right)}}{2n^{\rm U}_{k}}}+\sum_{k=1}^{q}\left(\left(1-\pi_{k}-\bar{\pi}_{k}\right)L_{g}+1-\pi_{k}\right)
C_{\ell}\sqrt{\frac{\ln{\left(1/\delta\right)}}{2n^{\rm N}_{k}}} \nonumber\\
&+ \sum_{k=1}^{q}\left(\left(4-4\bar{\pi}_{k}\right)L_{g}L_{\ell}\mathfrak{R}_{n^{\rm U}_{k},p^{\rm U}_{k}}(\mathcal{F})+\left(\left(4-4\pi_{k}-4\bar{\pi}_{k}\right)L_{g}+4-4\pi_{k}\right)L_{\ell}\mathfrak{R}_{n^{\rm N}_{k},p^{\rm N}_{k}}(\mathcal{F})\right). \label{prof_lemma_ineq_4} 
\end{align}
Then, we have 
\begin{align}
&\sup _{f_{1}, f_{2}, \ldots, f_{q} \in \mathcal{F}}\left|R(f_{1}, f_{2}, \ldots, f_{q})-\widetilde{R}(f_{1}, f_{2}, \ldots, f_{q})\right| \nonumber \\
=&\sup _{f_{1}, f_{2}, \ldots, f_{q} \in \mathcal{F}}\left|R(f_{1}, f_{2}, \ldots, f_{q})-\mathbb{E}\left[\widetilde{R}(f_{1}, f_{2}, \ldots, f_{q})\right]\right. \nonumber \\
&\left.+\mathbb{E}\left[\widetilde{R}(f_{1}, f_{2}, \ldots, f_{q})\right]-\widetilde{R}(f_{1}, f_{2}, \ldots, f_{q})\right| \nonumber \\
\leq&\sup _{f_{1}, f_{2}, \ldots, f_{q} \in \mathcal{F}}\left|R(f_{1}, f_{2}, \ldots, f_{q})-\mathbb{E}\left[\widetilde{R}(f_{1}, f_{2}, \ldots, f_{q})\right]\right| \nonumber \\
&+\sup _{f_{1}, f_{2}, \ldots, f_{q} \in \mathcal{F}}\left|\mathbb{E}\left[\widetilde{R}(f_{1}, f_{2}, \ldots, f_{q})\right]-\widetilde{R}(f_{1}, f_{2}, \ldots, f_{q})\right|. \label{prof_lemma_ineq_5}
\end{align}
Combining Inequality~(\ref{prof_lemma_ineq_5}) with Inequality~(\ref{prof_lemma_ineq_4}) and Inequality~(\ref{bias_ineq}), the proof is completed.
\end{proof}
We present a more complete version of Theorem~\ref{corrected_eeb} and its proof.
\begin{theorem}
Based on the above assumptions, for any $\delta>0$, the following inequality holds with probability at least $1-\delta$:
\begin{align}
&R(\widetilde{f}_{1}, \widetilde{f}_{2}, \ldots, \widetilde{f}_{q}) - R(f_{1}^{*},f_{2}^{*},\ldots,f_{q}^{*})\leq \sum_{k=1}^{q}\left(4-4\bar{\pi}_{k}-2\pi_{k}\right)\left(L_{g}+1\right)C_{\ell}\Delta_{k} \nonumber \\
&+ \sum_{k=1}^{q}\left( \left(1-\bar{\pi}_{k}\right)C_{\ell}L_{g}\sqrt{\frac{2\ln{\left(1/\delta\right)}}{n^{\rm U}_{k}}}+\left(\left(1-\pi_{k}-\bar{\pi}_{k}\right)L_{g}+1-\pi_{k}\right)
C_{\ell}\sqrt{\frac{2\ln{\left(1/\delta\right)}}{n^{\rm N}_{k}}}\right) \nonumber\\
&+ \sum_{k=1}^{q}\left(\left(8-8\bar{\pi}_{k}\right)L_{g}L_{\ell}\mathfrak{R}_{n^{\rm U}_{k},p^{\rm U}_{k}}(\mathcal{F})+\left(\left(8-8\pi_{k}-8\bar{\pi}_{k}\right)L_{g}+8-8\pi_{k}\right)L_{\ell}\mathfrak{R}_{n^{\rm N}_{k},p^{\rm N}_{k}}(\mathcal{F})\right).\nonumber 
\end{align}
\end{theorem}
\begin{proof} 
\begin{align}
&R(\widetilde{f}_{1}, \widetilde{f}_{2}, \ldots, \widetilde{f}_{q}) - R(f_{1}^{*},f_{2}^{*},\ldots,f_{q}^{*}) \nonumber \\
=&R(\widetilde{f}_{1}, \widetilde{f}_{2}, \ldots, \widetilde{f}_{q})-\widetilde{R}(\widetilde{f}_{1}, \widetilde{f}_{2}, \ldots, \widetilde{f}_{q})+\widetilde{R}(\widetilde{f}_{1}, \widetilde{f}_{2}, \ldots, \widetilde{f}_{q})-\widetilde{R}(f_{1}^{*},f_{2}^{*},\ldots,f_{q}^{*}) \nonumber \\
&+\widetilde{R}(f_{1}^{*},f_{2}^{*},\ldots,f_{q}^{*})- R(f_{1}^{*},f_{2}^{*},\ldots,f_{q}^{*}) \nonumber \\
\leq & R(\widetilde{f}_{1}, \widetilde{f}_{2}, \ldots, \widetilde{f}_{q})-\widetilde{R}(\widetilde{f}_{1}, \widetilde{f}_{2}, \ldots, \widetilde{f}_{q})+\widetilde{R}(f_{1}^{*},f_{2}^{*},\ldots,f_{q}^{*})- R(f_{1}^{*},f_{2}^{*},\ldots,f_{q}^{*}) \nonumber \\
\leq& 2\sup _{f_{1}, f_{2}, \ldots, f_{q} \in \mathcal{F}}\left|R(f_{1}, f_{2}, \ldots, f_{q})-\widetilde{R}(f_{1}, f_{2}, \ldots, f_{q})\right|. \label{ineq_of_thm4}
\end{align}
The first inequality is deduced because $(\widetilde{f}_{1}, \widetilde{f}_{2}, \ldots, \widetilde{f}_{q})$ is the minimizer of $\widetilde{R}(f_1, f_2, \ldots, f_q)$. Combining Inequality~(\ref{ineq_of_thm4}) and Lemma~\ref{lemma_of_thm4}, the proof is completed.
\end{proof}
\section{Details of Experimental Setup}\label{exp_setup}
\subsection{Details of Synthetic Benchmark Datasets}
We considered the single complementary-label setting and similar results could be observed with multiple complementary labels.

For the ``uniform'' setting, a label other than the ground-truth label was sampled randomly following the uniform distribution to be the complementary label.

For the ``biased-a'' and ``biased-b'' settings, we adopted the following row-normalized transition matrices of $p(\bar{y}|y)$ to generate complementary labels: \\
\begin{eqnarray}
\textrm{biased-a:}
\begin{bmatrix}
0&0.250&0.043&0.040&0.043&0.040&0.250&0.040&0.250&0.043 \\
0.043&0&0.250&0.043&0.040&0.043&0.040&0.250&0.040&0.250 \\
0.250&0.043&0&0.250&0.043&0.040&0.043&0.040&0.250&0.040 \\
0.040&0.250&0.043&0&0.250&0.043&0.040&0.043&0.040&0.250 \\
0.250&0.040&0.250&0.043&0&0.250&0.043&0.040&0.043&0.040 \\
0.040&0.250&0.040&0.250&0.043&0&0.250&0.043&0.040&0.043 \\
0.043&0.040&0.250&0.040&0.250&0.043&0&0.250&0.043&0.040 \\
0.040&0.043&0.040&0.250&0.040&0.250&0.043&0&0.250&0.043 \\
0.043&0.040&0.043&0.040&0.250&0.040&0.250&0.043&0&0.250 \\
0.250&0.043&0.040&0.043&0.040&0.250&0.040&0.250&0.043&0 \nonumber
\end{bmatrix},\\
\textrm{biased-b:}
\begin{bmatrix}
0&0.220&0.080&0.033&0.080&0.033&0.220&0.033&0.220&0.080 \\
0.080&0&0.220&0.080&0.033&0.080&0.033&0.220&0.033&0.220 \\
0.220&0.080&0&0.220&0.080&0.033&0.080&0.033&0.220&0.033 \\
0.033&0.220&0.080&0&0.220&0.080&0.033&0.080&0.033&0.220 \\
0.220&0.033&0.220&0.080&0&0.220&0.080&0.033&0.080&0.033 \\
0.033&0.220&0.033&0.220&0.080&0&0.220&0.080&0.033&0.080 \\
0.080&0.033&0.220&0.033&0.220&0.080&0&0.220&0.080&0.033 \\
0.033&0.080&0.033&0.220&0.033&0.220&0.080&0&0.220&0.080 \\
0.080&0.033&0.080&0.033&0.220&0.033&0.220&0.080&0&0.220 \\
0.220&0.080&0.033&0.080&0.033&0.220&0.033&0.220&0.080&0 \nonumber
\end{bmatrix}.
\end{eqnarray}
For each example, we sample a complementary label from a multinomial distribution parameterized by the row vector of the transition matrix indexed by the ground-truth label. 

For the ``SCAR-a'' and ``SCAR-b'' settings, we followed the generation process in Section 3.1 with the following class priors of complementary labels:
\begin{align}
\textrm{SCAR-a:} &\left[0.05, 0.05, 0.2, 0.2, 0.1, 0.1, 0.05, 0.05, 0.1, 0.1\right], \nonumber \\
\textrm{SCAR-b:} &\left[0.1, 0.1, 0.2, 0.05, 0.05, 0.1, 0.1, 0.2, 0.05, 0.05\right]. \nonumber 
\end{align}
We repeated the sampling procedure to ensure that each example had a single complementary label.
\subsection{Descriptions of Compared Approaches}
The compared methods in the experiments on synthetic benchmark datasets:
\begin{itemize}[leftmargin=1em, itemsep=1pt, topsep=1pt, parsep=-1pt]
\item PC~\citep{ishida2017learning}: A risk-consistent complementary-label learning approach using the pairwise comparison loss.
\item NN~\citep{ishida2019complementary}: A risk-consistent complementary-label learning approach using the non-negative risk estimator.
\item GA~\citep{ishida2019complementary}: A variant of the non-negative risk estimator of complementary-label learning by using the gradient ascent technique.
\item L-UW~\citep{gao2021discriminative}: A discriminative  approach by minimizing the outputs corresponding to complementary labels.
\item L-W~\cite{gao2021discriminative}: A weighted loss based on L-UW by considering the prediction uncertainty.
\item OP~\citep{liu2023consistent}: A classifier-consistent complementary-label learning approach by minimizing the outputs of complementary labels.
\end{itemize}
The compared methods in the experiments on real-world benchmark datasets:
\begin{itemize}[leftmargin=1em, itemsep=1pt, topsep=1pt, parsep=-1pt]
\item CC~\citep{feng2020provably}: A classifier-consistent partial-label learning approach based on the assumption of uniform distribution of partial labels.
\item PRODEN~\citep{lv2020progressive}: A risk-consistent partial-label learning approach using the self-training strategy to identify the ground-truth labels.
\item EXP~\citep{feng2020learning}: A classifier-consistent multiple complementary-label learning approach by using the exponential loss function.
\item MAE~\citep{feng2020learning}: A classifier-consistent multiple complementary-label learning approach by using the Mean Absolute Error loss function.
\item Phuber-CE~\citep{feng2020learning}: A classifier-consistent multiple complementary-label learning approach by using the Partially Huberised Cross Entropy loss function.
\item LWS~\citep{wen2021leveraged}: A partial-label learning approach by leveraging a weight to account for the tradeoff between losses on partial and non-partial labels.
\item CAVL~\citep{zhang20222exploiting}: A partial-label learning approach by using the class activation value to identify the true labels.
\item IDGP~\citep{qiao2023decompositional}: An instance-dependent partial-label learning approach by modeling the generation process of partial labels.
\item POP~\citep{xu2023progressive}: A  partial-label learning approach by progressively purifying candidate label sets.
\end{itemize}
\subsection{Details of Models}
For CIFAR-10, we used 34-layer ResNet~\citep{he2016deep} and 22-layer DenseNet~\citep{huang2017densely} as the model architectures. For the other three datasets, we used a multilayer perceptron~(MLP) with a hidden layer of width 500 equipped with the ReLU~\citep{nair2010rectified} activation function and 5-layer LeNet~\citep{lecun1998gradient} as the model architectures. 

For CLCIFAR-10 and CLCIFAR-20, we adopted the same data augmentation techniques for all the methods, including random horizontal flipping and random cropping. We used 34-layer ResNet~\citep{he2016deep} and 22-layer DenseNet~\citep{huang2017densely} as the model architectures.

\section{More Experimental Results}\label{more_res}
\subsection{Experimental Results on CIFAR-10}
Table~\ref{exp_res_cifar10} shows the experimental results on CIFAR-10 with synthetic complementary labels.
\begin{table*}[htbp]
\small
\caption{Classification accuracy~(mean$\pm$std) of each method on CIFAR-10 with a single complementary label. The best performance is shown in bold~(pairwise \emph{t}-test at the 0.05 significance level).
}\label{exp_res_cifar10}
\setlength{\tabcolsep}{3pt}
\centering
\begin{tabular}{l cc cc cc cc cc} 
\toprule[1pt]   
Setting &\multicolumn{2}{c}{Uniform}&\multicolumn{2}{c}{Biased-a}&\multicolumn{2}{c}{Biased-b}&\multicolumn{2}{c}{SCAR-a}&\multicolumn{2}{c}{SCAR-b}\\ 
\midrule Model&  ResNet & DenseNet & ResNet & DenseNet &ResNet & DenseNet& ResNet & DenseNet & ResNet & DenseNet\\
\midrule PC& \makecell{14.33 \\ $\pm$\scriptsize{0.73 }}
& \makecell{17.44 \\ $\pm$\scriptsize{0.52 }}
& \makecell{25.46 \\ $\pm$\scriptsize{0.69 }}
& \makecell{34.01 \\ $\pm$\scriptsize{1.47 }}
& \makecell{23.04 \\ $\pm$\scriptsize{0.33 }}
& \makecell{29.27 \\ $\pm$\scriptsize{1.05 }}
& \makecell{14.94 \\ $\pm$\scriptsize{0.88 }}
&\makecell{17.11 \\ $\pm$\scriptsize{0.87 }}
& \makecell{17.16 \\ $\pm$\scriptsize{0.86 }}
& \makecell{21.14 \\ $\pm$\scriptsize{1.34 }} \\ 
\midrule NN& \makecell{19.90 \\ $\pm$\scriptsize{0.73 }}
& \makecell{30.55 \\ $\pm$\scriptsize{1.01 }}
& \makecell{24.88 \\ $\pm$\scriptsize{1.01 }}
& \makecell{24.48 \\ $\pm$\scriptsize{1.50 }}
& \makecell{26.59 \\ $\pm$\scriptsize{1.33 }}
& \makecell{24.51 \\ $\pm$\scriptsize{1.24 }}
& \makecell{21.11 \\ $\pm$\scriptsize{0.94 }}
&\makecell{29.48 \\ $\pm$\scriptsize{1.05 }}
& \makecell{23.56 \\ $\pm$\scriptsize{1.25 }}
& \makecell{30.67 \\ $\pm$\scriptsize{0.73 }} \\ 
\midrule GA& \textbf{\makecell{37.59 \\ $\pm$\scriptsize{1.76 }}}
& \textbf{\makecell{46.86 \\ $\pm$\scriptsize{0.84 }}}
& \makecell{20.01 \\ $\pm$\scriptsize{1.96 }}
& \makecell{22.41 \\ $\pm$\scriptsize{1.33 }}
& \makecell{16.74 \\ $\pm$\scriptsize{2.64 }}
& \makecell{21.48 \\ $\pm$\scriptsize{1.46 }}
& \textbf{\makecell{24.17 \\ $\pm$\scriptsize{1.32 }}}
&\makecell{29.04 \\ $\pm$\scriptsize{1.84 }}
& \makecell{23.47 \\ $\pm$\scriptsize{1.30 }}
& \makecell{30.72 \\ $\pm$\scriptsize{1.44 }} \\ 
\midrule L-UW& \makecell{19.58 \\ $\pm$\scriptsize{1.77 }}
& \makecell{17.25 \\ $\pm$\scriptsize{3.03 }}
& \makecell{24.83 \\ $\pm$\scriptsize{2.67 }}
& \makecell{29.46 \\ $\pm$\scriptsize{1.03 }}
& \makecell{20.73 \\ $\pm$\scriptsize{2.41 }}
& \makecell{25.41 \\ $\pm$\scriptsize{2.61 }}
& \makecell{14.56 \\ $\pm$\scriptsize{2.71 }}
&\makecell{10.69 \\ $\pm$\scriptsize{0.94 }}
& \makecell{10.39 \\ $\pm$\scriptsize{0.50 }}
& \makecell{10.04 \\ $\pm$\scriptsize{0.09 }} \\ 
\midrule L-W& \makecell{18.05 \\ $\pm$\scriptsize{3.02 }}
& \makecell{13.97 \\ $\pm$\scriptsize{2.55 }}
& \makecell{22.65 \\ $\pm$\scriptsize{2.70 }}
& \makecell{27.64 \\ $\pm$\scriptsize{0.80 }}
& \makecell{22.70 \\ $\pm$\scriptsize{2.33 }}
& \makecell{24.86 \\ $\pm$\scriptsize{1.34 }}
& \makecell{13.72 \\ $\pm$\scriptsize{2.60 }}
&\makecell{10.00 \\ $\pm$\scriptsize{0.00 }}
& \makecell{10.25 \\ $\pm$\scriptsize{0.49 }}
& \makecell{10.00 \\ $\pm$\scriptsize{0.00 }} \\ 
\midrule OP& \makecell{23.78 \\ $\pm$\scriptsize{2.80 }}
& \makecell{39.32 \\ $\pm$\scriptsize{2.46 }}
& \makecell{29.47 \\ $\pm$\scriptsize{2.71 }}
& \makecell{41.99 \\ $\pm$\scriptsize{1.54 }}
& \makecell{25.60 \\ $\pm$\scriptsize{4.18 }}
& \makecell{39.61 \\ $\pm$\scriptsize{2.26 }}
& \makecell{17.55 \\ $\pm$\scriptsize{1.38 }}
&\makecell{27.12 \\ $\pm$\scriptsize{1.17 }}
& \makecell{20.08 \\ $\pm$\scriptsize{2.96 }}
& \makecell{27.24 \\ $\pm$\scriptsize{2.62 }} \\ 
\midrule SCARCE& \textbf{\makecell{35.63  \\ $\pm$\scriptsize{3.23  }}}
& \makecell{42.65  \\ $\pm$\scriptsize{2.00  }}
& \textbf{\makecell{39.70  \\ $\pm$\scriptsize{3.79  }}}
& \textbf{\makecell{51.42  \\ $\pm$\scriptsize{1.81  }}}
& \textbf{\makecell{37.82  \\ $\pm$\scriptsize{2.72  }}}
& \textbf{\makecell{50.52  \\ $\pm$\scriptsize{2.18  }}}
& \textbf{\makecell{29.04  \\ $\pm$\scriptsize{3.70  }}}
&\textbf{\makecell{36.38  \\ $\pm$\scriptsize{2.56  }}}
& \textbf{\makecell{35.71  \\ $\pm$\scriptsize{1.16  }}}
& \textbf{\makecell{38.43  \\ $\pm$\scriptsize{0.85  }}} \\ 
\bottomrule[1pt]   
\end{tabular}
\end{table*}
\end{document}